\documentclass{article} 
\usepackage{arxiv,times}

\usepackage{amsmath,amsfonts,bm}

\def\eqref#1{equation~\ref{#1}}

\def\1{\bm{1}}

\DeclareMathAlphabet{\mathsfit}{\encodingdefault}{\sfdefault}{m}{sl}
\SetMathAlphabet{\mathsfit}{bold}{\encodingdefault}{\sfdefault}{bx}{n}

\def\gD{{\mathcal{D}}}
\def\gE{{\mathcal{E}}}
\def\gF{{\mathcal{F}}}

\def\gR{{\mathcal{R}}}

\newcommand{\R}{\mathbb{R}}

\renewcommand{\eqref}[1]{(\ref{#1})} 

\usepackage{booktabs}
\usepackage{tabularx}
\newcolumntype{Y}{>{\raggedright\arraybackslash}X} 
\usepackage{amssymb}
\usepackage{amsthm}
\usepackage{mathtools}
\usepackage{graphicx}
\usepackage{wrapfig}
\usepackage{subcaption}
\usepackage{blkarray}
\newcommand{\sd}[1]{{\scriptstyle\pm#1}} 
\theoremstyle{definition}
\newtheorem{definition}{Definition}
\newtheorem{theorem}{Theorem}
\newtheorem{proposition}{Proposition}

\usepackage{hyperref}
\usepackage{url}

\title{Dirac-Interconnected Neural Elements: \\Discovering Modularity in Physical Systems Without Reduction}

\author{Anonymous authors}

\makeatletter
\newcommand{\silentthanks}[1]{{\let\@makefnmark\relax\thanks{#1}}}
\makeatother
\author{%
  Reiho Li$^{*1}$%
  \silentthanks{Correspondence to~\texttt{li.reiho.c0@elms.hokudai.ac.jp},~\texttt{matsubara@ist.hokudai.ac.jp}}
  Razmik Arman Khosrovian$^1$\,
  Takaharu Yaguchi$^{2,3}$\\
  \bf Hiroaki Yoshimura$^4$\,
  Takashi Matsubara$^{*1,3}$\\
$^1$Hokkaido University\,
$^2$Kyushu University\,
$^3$RIKEN\,
$^4$Waseda University\,
}

\iclrfinalcopy 
\begin{document}

\maketitle
\lhead{Preprint}

\begin{abstract}
    Deep learning has shown remarkable success in the data-driven modeling of dynamical systems.
    Much of its success is attributed not to the flexibility of neural networks but to inductive biases based on physical prior knowledge, such as energy conservation and symplecticity.
    However, existing methods do not fully exploit the fact that real-world physical systems are interconnections of components.
    Some methods require the interconnection to be known a priori, while others assume the system to be reducible to an ordinary differential equation (ODE) and learn only the reduced ODE, discarding the algebraic constraints imposed by the interconnection.
    Here, we propose Dirac-interconnected neural elements (DINEs), a neural network model that represents a physical system as a differential-algebraic equation (DAE), whose algebraic constraints are given by a Dirac structure in kernel representation.
    With DINEs, we simultaneously identify from data the interconnection among the components as a Dirac structure and learn the characteristics of the components as neural networks.
    This allows us to keep the learned subsystems in unreduced form and isolate or compose them to make a new system without retraining.
    Moreover, DINEs can handle partially observable systems.
    Experimental results demonstrate these capabilities on physical systems beyond the reach of existing methods.
\end{abstract}

\section{Introduction}
Deep learning has been successful in a wide range of domains~\citep{bishop2024deep,noor2025survey}, including the data-driven modeling of dynamical systems~\citep{wang2023scientific,legaard2023constructing}.
Many approaches approximate the dynamics directly from observed trajectories, as in neural ordinary differential equations (Neural ODEs)~\citep{chen2018neural}.
Since such black-box models often exhibit physically inconsistent behavior in long-term predictions, physical prior knowledge has been embedded into models as an inductive bias.
In particular, Hamiltonian, constrained Hamiltonian, and port-Hamiltonian neural networks are inspired by energy-based modeling, where an energy function is defined on the state space, and a single vector field is derived from it, ensuring the energetic behavior of the system by construction~\citep{greydanus2019hamiltonian,finzi2020simplifying,matsubara2020deep,desai2021port,moradi2026port}.

However, real-world physical systems are interconnections of components, such as masses, springs, and dampers in mechanical systems, and inductors, capacitors, and resistors in electrical circuits.
Their power exchange has often been described by a graphical notation, bond graphs~\citep{paynter1961analysis,karnopp2012system,tabatabaian2021engineering}.
The model becomes interpretable and its components become reusable if it identifies from observed data this internal structure, namely the characteristics of the components (i.e., their constitutive relations) and the interconnection among them.
Beyond prediction, such models also support the identification of physical parameters, diagnosis, and design.
Because the interconnection among the components gives rise to algebraic constraints, interconnected systems are in general described by differential-algebraic equations (DAEs)~\citep{duindam2009modeling,vanderschaft2014port}.
Neural port-Hamiltonian DAEs (N-pHDAEs)~\citep{nphdae} and port-Hamiltonian DAE (pH-DAE) identification~\citep{phdae} approximate interconnected systems with the interconnection known and fixed in advance.
Poisson-Dirac neural networks (PoDiNNs)~\citep{podinn} identify the interconnection from observed data, but only in a reduced form.
They assume that the system is reducible to an ODE and that all component states are observable, thereby discarding the algebraic constraints and excluding partially observable systems.
Table~\ref{tab:comparison} summarizes the comparison: no existing model both learns the interconnection and keeps its algebraic constraints.

\begin{table}[t]
    \centering
    \caption{Comparison with related methods.}
    \label{tab:comparison}
    \scriptsize
    \setlength{\tabcolsep}{3pt}
    \begin{tabular}{lcccc}
        \toprule
        \textbf{Model}                                     & {\bfseries\begin{tabular}[c]{@{}c@{}}Ensuring\\Energetic Behavior\end{tabular}}
                                                           & {\bfseries\begin{tabular}[c]{@{}c@{}}Identifying\\Interconnection\end{tabular}}
                                                           & {\bfseries\begin{tabular}[c]{@{}c@{}}Learning\\Constraints\end{tabular}}
                                                           & {\bfseries\begin{tabular}[c]{@{}c@{}}Learning from\\Partial Observation\end{tabular}}                                          \\
        \midrule
        Neural ODEs~\citep{chen2018neural}                 & --                                                                                    & --         & --           & --         \\
        HNNs, LNNs, SymODENs, pHNNs $^\dagger$             & \checkmark                                                                            & --         & --           & --         \\
        OE-pHNNs~\citep{moradi2026port}                    & \checkmark                                                                            & --         & --           & \checkmark \\
        CHNNs, CLNNs~\citep{finzi2020simplifying}          & \checkmark                                                                            & --         & given $^\S$  & --         \\
        Compositional pHNNs~\citep{neary2023compositional} & \checkmark                                                                            & given      & --           & --         \\
        N-pHDAEs~\citep{nphdae}                            & \checkmark                                                                            & given      & given        & --         \\
        pH-DAE identification~\citep{phdae}                & \checkmark                                                                            & given      & given        & \checkmark \\
        Latent ODEs, CT-SUBNET $^\ddagger$                 & --                                                                                    & --         & if reducible & \checkmark \\
        PoDiNNs~\citep{podinn}                             & \checkmark                                                                            & \checkmark & if reducible & --         \\
        \midrule
        DINEs (ours)                                       & \checkmark                                                                            & \checkmark & \checkmark   & \checkmark \\
        \bottomrule
    \end{tabular}\\\raggedright
    $^\dagger$ \citet{greydanus2019hamiltonian,cranmer2020lagrangian,zhong2020symplectic,zhong2020dissipative,desai2021port}. $^\ddagger$ \citet{rubanova2019latent,beintema2023continuous}. $^\S$ Available only for holonomic constraints.
\end{table}

\begin{wrapfigure}{r}{0.34\textwidth}
    \vspace*{-4mm}
    \centering
    \includegraphics[width=\linewidth]{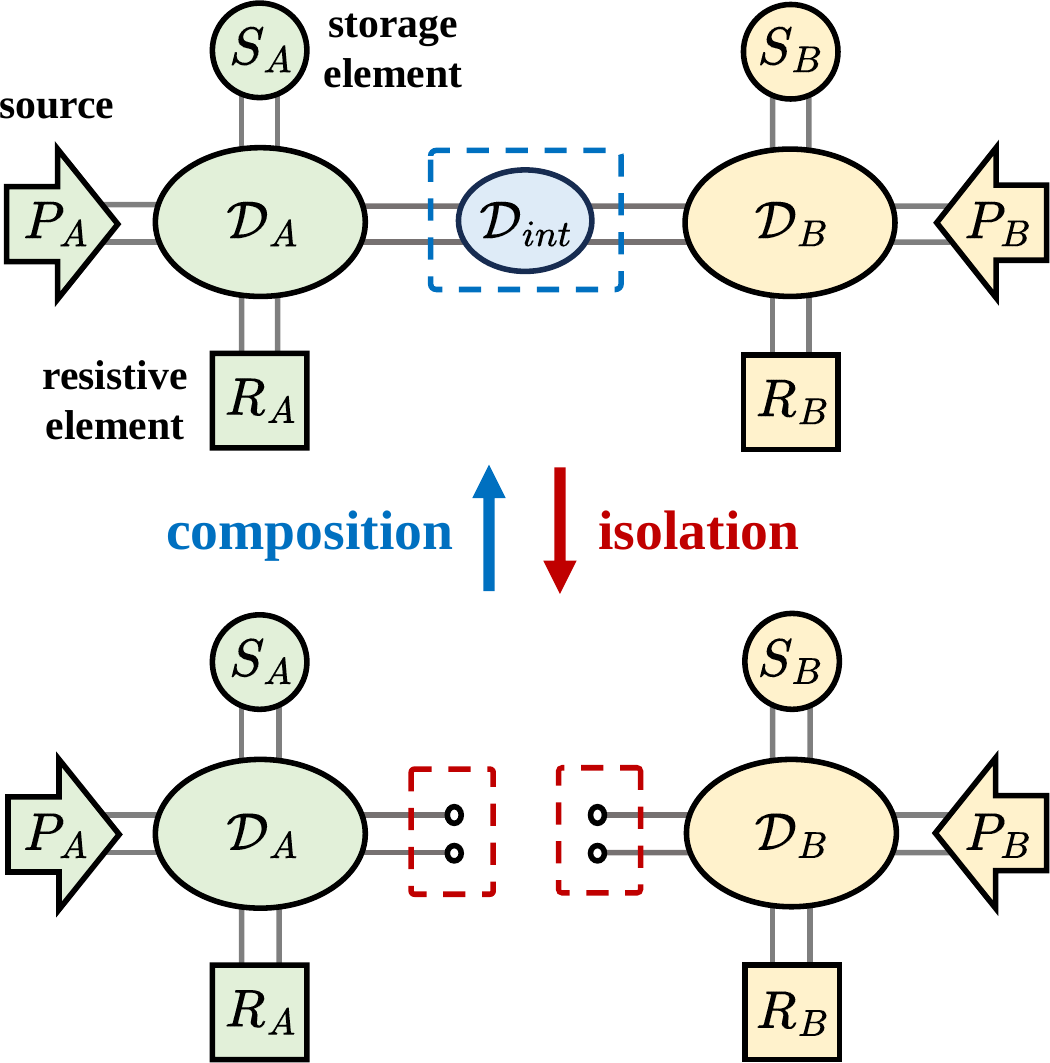}
    \vspace*{-4mm}
    \caption{Conceptual diagram of DINEs.}
    \label{fig:concept}
    \vspace*{-3mm}
\end{wrapfigure}
Given the above, we propose Dirac-interconnected neural elements (DINEs), a neural network model that represents a physical system as a DAE whose algebraic constraints are given by a Dirac structure in kernel representation (see Figure~\ref{fig:concept})~\citep{duindam2009modeling,vanderschaft2014port}.
Our contributions are as follows:
(i) By minimizing the DAE residual, DINEs simultaneously learn the interconnection as a Dirac structure and the constitutive relations of the components as neural networks.
(ii) Since the system is not reduced to an ODE, DINEs can isolate or compose the learned subsystems without retraining.
(iii) DINEs also handle partially observable systems, in which some component states are unobserved.
(iv) Experimental results demonstrate these capabilities on physical systems beyond the reach of existing methods.

\section{Background}
\label{sec:background}
\subsection{Port-based modeling and Dirac structures}
\label{sec:port}
\textbf{Port-based modeling.}
Port-based modeling represents a physical system as an interconnection of elements, each of which idealizes a basic physical behavior~\citep{duindam2009modeling,vanderschaft2014port}.
Each element has \emph{ports}, through which it exchanges energy with the others.
The exchange is described by a pair of \emph{port variables}, a \emph{flow} $f$ (e.g., velocity or current) and an \emph{effort} $e$ (e.g., force or voltage), whose duality product gives the power flowing out of the element through the port; their physical nature depends on the domain.
A C-type storage element (C; e.g., a spring or a capacitor) has a state $q$ and stores energy $H(q)$ with $f = -\dot{q}$ and $e = \partial H / \partial q$, and an I-type storage element (I; e.g., a mass or an inductor) has a state $p$ and stores energy $H(p)$ with $e = -\dot{p}$ and $f = \partial H / \partial p$; in both cases $\langle e \mid f \rangle = -\dot{H}$ holds, and these equations are called the energy-storage relations.
A resistive element (R; e.g., a damper or a resistor) is defined by a resistive relation $\gR(f, e) = 0$ that satisfies $e f \le 0$ and thus dissipates energy.
A source (Sf or Se) prescribes a flow or an effort as a function of time.
The interconnection is a multiport element that merely routes energy, that is, it is \emph{power continuous} with zero net power at all times: in bond-graph theory, it is called the junction structure and is composed of 0-junctions and 1-junctions, which generalize Kirchhoff's current and voltage laws, transformers (TF), and gyrators (GY)~\citep{karnopp2012system}.
See Appendix~\ref{app:domains} for details.

\textbf{Dirac structure.}
The interconnection of port-based modeling is formalized as a Dirac structure~\citep{vanderschaft1995hamiltonian,dalsmo1998representations}, a geometric object introduced by \citet{courant1990dirac} and \citet{dorfman1993dirac}.
Let $\gF$ be an $n$-dimensional linear space, whose element $f$ is the vector of the flows at the ports.
Let $\gE := \gF^*$ be its dual space, whose element $e$ is the vector of the efforts at the ports.
For $(f, e) \in \gF \times \gE$, the total power is the duality product $\langle e \mid f \rangle$, which is equal to $e^\top f$ when $f$ and $e$ are represented as column vectors.

\begin{definition}[Dirac structure; \citealp{courant1990dirac,duindam2009modeling,vanderschaft2014port}]
    \label{def:dirac}
    A Dirac structure on $\gF \times \gE$ is a subspace $\gD \subset \gF \times \gE$ such that
    (i) $\langle e \mid f \rangle = 0$ for all $(f, e) \in \gD$, and
    (ii) $\dim \gD = \dim \gF$.
\end{definition}
Every Dirac structure $\gD \subset \gF \times \gE$ can be represented in the kernel representation~\citep[][Sec.~5.1]{dalsmo1998representations,vanderschaft2014port}.
With linear coordinates for $\gF$ and the dual basis for $\gE$, it is given by
\begin{equation}
    \gD
    = \left\{ (f, e) \in \gF \times \gE
    \;\middle|\; F f + E e = 0 \right\},
    \label{eq:kernel}
\end{equation}
where $F$ and $E$ are $n \times n$ matrices satisfying
\begin{equation}
    E F^\top + F E^\top = 0 \mbox{ and }\operatorname{rank} [\, F \mid E \,] = n.
    \label{eq:conditions}
\end{equation}
Conversely, any pair of matrices satisfying these two conditions defines a Dirac structure by Eq.~\eqref{eq:kernel}.

\textbf{Dynamics on a Dirac structure.}
Each of the 0-junction, the 1-junction, the transformer, and the gyrator defines a Dirac structure, and the composition of Dirac structures is again a Dirac structure.
Hence, the junction structure is a Dirac structure~\citep{duindam2009modeling,vanderschaft2014port}, which we denote by $\gD$.
Each row of $F f + E e = 0$ is a linear constraint that the interconnection imposes on the port variables.
The entries of $F$ and $E$ represent the interconnection, including the ratios of the transformers and gyrators.

The dynamics is defined by the condition that the port variables belong to $\gD$ at all times~\citep{duindam2009modeling,vanderschaft2014port}:
\begin{equation}
    (f(t), e(t)) \in \gD \;\text{ for all } t,
    \quad\text{that is,}\quad
    F f(t) + E e(t) = 0 \;\text{ for all } t.
    \label{eq:dirac-dae}
\end{equation}
Here, $f = (f_S, f_R, f_P)$ and $e = (e_S, e_R, e_P)$ are the port variables of the storage elements, the resistive elements, and the sources, called the storage, resistive, and source ports, respectively.
The numbers of these ports are $n_S$, $n_R$, and $n_P$, respectively, with $n = n_S + n_R + n_P$.
We partition $F$ and $E$ in Eq.~\eqref{eq:kernel} in the same way as $f$ and $e$, that is, $F = [\, F_S \mid F_R \mid F_P \,]$ and $E = [\, E_S \mid E_R \mid E_P \,]$.
The port variables are oriented out of the elements and thus into the Dirac structure, so that condition (i) of Definition~\ref{def:dirac} states that the net power flowing into the interconnection is zero.

Let $z_i$ be the state of the $i$-th storage element, which is $q_i$ for C-type and $p_i$ for I-type, and let $z = (z_1, \dots, z_{n_S}) \in \R^{n_S}$; only the storage elements have states, so $z$ is the state of the entire system.
The Hamiltonian $H : \R^{n_S} \to \R$ is the total stored energy, that is, the sum of the energies of all the storage elements, $H(z) = \sum_{i=1}^{n_S} H_i(z_i)$, where $H_i$ is the energy of the $i$-th storage element.
The resistive ports satisfy a resistive relation between $f_R$ and $e_R$ such that $\langle e_R \mid f_R \rangle \le 0$, that is, the resistive elements only absorb power.
For each source port, the prescribed variable (the flow of an Sf element or the effort of an Se element) is an external input, and its conjugate variable is determined by the system as an output; we denote the vectors of the inputs and the outputs by $u$ and $y$, respectively.

The constitutive relations are thus given by the Hamiltonian $H$ and the resistive relation $\gR$.
Substituting the constitutive relations into Eq.~\eqref{eq:dirac-dae} yields, in general, a differential-algebraic equation (DAE) for the state $z$ and the remaining port variables with the input $u$.

\textbf{Reduction to an ODE.}
Under certain conditions, the DAE can be reduced to an ODE in $z$.
Exchanging the flow and the effort of every I-type port, every Sf port, and some resistive ports preserves the conditions of Eq.~\eqref{eq:conditions}~\citep[Sec.~5.3]{vanderschaft2014port}; we denote the resulting variables with $\tilde{\cdot}$.
Then, the flows of all the storage ports are the time derivatives of the states up to sign, $\dot{z}=-\tilde{f}_S$.
If $\tilde{F}$ is invertible, Eq.~\eqref{eq:kernel} is equivalent to $\tilde{f} = B \tilde{e}$ with the skew-symmetric matrix $B := -\tilde{F}^{-1} \tilde{E}$, which is called the input-output representation of the Dirac structure.
This form is an ODE in $z$ once the remaining port variables can be eliminated through the constitutive relations, and is what PoDiNNs employ~\citep{podinn}.

Under further assumptions, eliminating the resistive port variables yields the input-state-output form of port-Hamiltonian systems~\citep{maschke1992port,vanderschaft1995hamiltonian,vanderschaft2014port},
\begin{equation}
    \dot{z} = \big( J - R \big)\, \nabla H(z) + G\, u,
    \qquad
    y = G^\top \nabla H(z),
    \qquad
    R := G_R \hat{R} G_R^\top,
    \label{eq:reduction}
\end{equation}
where $u = \tilde{e}_P$, $y = \tilde{f}_P$, $\hat{R}$ is the coefficient of the linear resistive relations, and a skew-symmetric matrix $J = -J^\top$, the resistive connection $G_R$, and the gain $G$ are blocks of $B$ up to sign.
This form is what many existing methods employ (as discussed below).

We refer to these transformations and eliminations of variables as \emph{reduction}, which has two costs.
(1) It mixes the interconnection with the elements, since $R$ combines the resistive relations $\hat{R}$ with the blocks $G_R$ that connect them and cannot be uniquely decomposed into the two.
(2) It assumes that an ODE exists, but it does not in general.
For example, Eq.~\eqref{eq:reduction} cannot describe nonlinear frictions, and both forms are not applicable to the balanced force of several nonlinear springs, and a capacitor connected with a voltage source in parallel (see also Appendix~\ref{app:reduction}).

\subsection{Related work}
\label{sec:related}

\textbf{Learning the dynamics as a single ODE.}
Neural ODEs~\citep{chen2018neural} approximate a general ODE of the form $\dot{z} = \varphi_\theta(z)$ with a neural network, where no mechanism prevents the learned model from creating or destroying energy in long-term predictions.
Hamiltonian neural networks (HNNs)~\citep{greydanus2019hamiltonian} and their extensions, Symplectic ODE-Nets (SymODENs)~\citep{zhong2020symplectic}, neural symplectic forms~\citep{chen2021neural}, Dissipative SymODENs~\citep{zhong2020dissipative}, port-Hamiltonian neural networks (pHNNs)~\citep{desai2021port}, pseudo-Hamiltonian neural networks~\citep{eidnes2023pseudo}, and output-error pHNNs (OE-pHNNs)~\citep{moradi2026port}, are special cases of Eq.~\eqref{eq:reduction} that learn $H_\theta$ with $J$, $R$, and $G$ fixed or learned.
Lagrangian neural networks (LNNs) and deep Lagrangian networks~\citep{lutter2019deep,cranmer2020lagrangian}, constrained HNNs and LNNs (CHNNs and CLNNs)~\citep{finzi2020simplifying}, which add holonomic constraints given in advance, are considered also as the special cases with designed $J$.
See Appendix~\ref{app:reduction} for the correspondence.
In the methods, $H_\theta$ is a single function of the whole state $z$, or at most separated into the kinetic and potential energies but not into the energies of the individual elements.
The resistive elements appear only in $R$, so that neither the elements nor their interconnection can be identified from the learned model.
The constraints imposed by the interconnection are either eliminated manually before training by the choice of $z$ or given in advance.

\textbf{Learning the components under a given interconnection.}
Several methods learn the subsystems or elements separately, while the interconnection is specified before training, or at most fitted as a composition between ODE subsystems.
Compositional pHNNs~\citep{neary2023compositional} learn each subsystem in the form of Eq.~\eqref{eq:reduction} for a given $J_i$ and compose them through an additional skew-symmetric matrix $J_{\mathrm{int}}$, which is given or learned from data of the composite system as a matrix or a neural network; the composite system is again an ODE of the same form with $H = \sum_i H_{\theta,i}$.
N-pHDAEs~\citep{nphdae} employ the DAE of an electrical network in the descriptor form~\citep{beattie2018linear,mehrmann2023control} with the circuit topology known and all of its variables observable, that is, $F$ and $E$ of Eq.~\eqref{eq:dirac-dae} are given by Kirchhoff's laws and only the constitutive relations are learned; separately trained N-pHDAEs can be composed without additional training if the additional interconnection is given as a matrix.
\citet{phdae} identify a linear pH-DAE of the same family from input-output data with unobserved states.

\textbf{Identifying the interconnection from data.}
The closest to our work is PoDiNNs~\citep{podinn}, which employ the Dirac structure in the input-output representation $\tilde{f} = B_\theta \tilde{e}$.
Given the list of elements, they learn $B_\theta$ as a constant matrix, the resistive relations as explicit maps $\tilde{e}_R = \hat{\gR}_\theta(\tilde{f}_R)$, and $H_\theta$ under the assumption that all states are observed.
However, after reduction, the interconnection represented by $B_\theta$ depends on the elements' states or parameters in general.
\citet{podinn} themselves acknowledge in their Appendix~B.2 that, to represent the equal velocities of two rigidly coupled masses, $B_\theta$ depends on the masses and is no longer determined by the interconnection alone.
Also, PoDiNNs are inapplicable to some situations, such as the balanced force of several nonlinear springs (see Task~1 in Section~\ref{sec:experiments}).

\textbf{Learning with unobserved states.}
When only an observation $y$ of the system is available, the state has to be inferred along with the model.
The continuous-time subspace encoder network (CT-SUBNET)~\citep{beintema2021nonlinear,beintema2023continuous} learns a black-box state-space model, $\dot{z} = \varphi_\theta(z, u)$ and $y = h_\theta(z)$ with a predetermined state dimension $n_z$, where an encoder estimates the state at the start of each window of the time series from the past outputs and inputs and the prediction error over the window is minimized; Latent ODEs~\citep{rubanova2019latent} are similar with a probabilistic encoder and decoder.
OE-pHNNs and the pH-DAE identification replace the black-box model by Eq.~\eqref{eq:reduction} and by the linear DAE, respectively, but limit the observation $y$ to the output conjugate to the input, so that other quantities, such as the currents of the resistors inside a circuit, are not available as observations.
In all of these methods, the unobserved state is inferred by the encoder only at the start of each window and is obtained after integrating the model forward in time.

\section{Dirac-interconnected neural elements}
\label{sec:method}

\paragraph{Learnable elements and interconnection.}
\label{sec:model}
Here, we propose DINEs as a learnable model of Eq.~\eqref{eq:dirac-dae}, which is defined by constitutive relations and a Dirac structure.
For each storage element, we learn a differentiable function $H_{\theta,i}$ of its state $z_i$ by a neural network, and thus the total energy is $H_\theta(z) = \sum_{i=1}^{n_S} H_{\theta,i}(z_i)$.
For each resistive element, we learn the resistive relation $\gR_{\theta,j}(f_{R,j}, e_{R,j}) = 0$ by a neural network as an explicit map $e_{R,j} = \hat{\gR}_{\theta,j}(f_{R,j})$ from the flow to the effort.
Dissipativity $\langle e_R \mid f_R \rangle \le 0$ is not imposed for simplicity, while it is possible with a sign constraint on $\hat{\gR}_{\theta,j}$.
We suppose sources are prescribed as functions of time.

For the Dirac structure, the matrices $F,E$ should satisfy the two conditions of Eq.~\eqref{eq:conditions} for every value of the parameters.
To avoid constrained optimization, we suppose that the interconnection is separable and constant, and that the number $n_f$ of flow constraints (and thus the number $n_e=n-n_f$ of effort constraints) is given in advance.
A Dirac structure is \emph{separable} if the rows of Eq.~\eqref{eq:kernel} can be chosen so that each involves only flows or only efforts; it holds if the system has no gyrators, since junctions and transformers are separable and so is their composition.
The interconnection is constant if the ratios of its transformers do not depend on the state~\citep[Sec.~2]{vanderschaft2014port}.

\begin{theorem}[Exact parameterization of separable Dirac structures]\label{thm:dirac_parameterize}
    Let $n_f$ and $n_e$ be nonnegative integers with $n_f+n_e=n$, and let $W_0\in\R^{n\times n}$ be any fixed orthogonal matrix.
    For $\Omega\in\R^{n_f\times n_e}$, define $W(\Omega) = \exp\big(S(\Omega)\big)\, W_0$ with $S(\Omega) = \begin{bsmallmatrix} 0 & \Omega \\ -\Omega^\top & 0 \end{bsmallmatrix}$, let $W_f\in\R^{n_f\times n}$ and $W_e\in\R^{n_e\times n}$ consist of the first $n_f$ and last $n_e$ rows of $W(\Omega)$, respectively, and set $F = \begin{bsmallmatrix} W_f \\ 0 \end{bsmallmatrix}$ and $E = \begin{bsmallmatrix} 0 \\ W_e \end{bsmallmatrix}$.
    Denote the resulting kernel by $\gD(\Omega):=\{(f,e)\in\gF\times\gE\mid Ff+Ee=0\}.$
    Then, (i) For every $\Omega$, $(F,E)$ satisfies the conditions of Eq.~\eqref{eq:conditions}, and $\gD(\Omega)$ is a separable Dirac structure with exactly $n_f$ independent flow constraints.
    (ii) For every separable Dirac structure $\gD_\star$ on $\gF\times\gE$ with exactly $n_f$ independent flow constraints, there exists $\Omega\in\R^{n_f\times n_e}$ such that $\gD_\star\!=\!\gD(\Omega)$ for the same fixed $W_0$.
\end{theorem}

\begin{proposition}[Energy balance]\label{prop:energy}
    For every $\Omega$, $\theta$, input $u$, and solution $(z(t), f(t), e(t))$ of the DAE of DINEs,
    \begin{equation}\textstyle
        \frac{\mathrm{d}}{\mathrm{d}t} H_\theta(z(t)) = \langle e_R(t) \mid f_R(t) \rangle + \langle e_P(t) \mid f_P(t) \rangle.
        \label{eq:energy-balance}
    \end{equation}
    This means that the total energy is changed outside the storage elements.
    If, in addition, a sign constraint $f_{R,j} \hat{\gR}_{\theta,j}(f_{R,j}) \le 0$ is imposed on resistive relations, $\frac{\mathrm{d}}{\mathrm{d}t} H_\theta(z(t)) \le \langle e_P(t) \mid f_P(t) \rangle$ holds, which means that the total energy is dissipated.
\end{proposition}

See Appendix~\ref{app:proofs} for the proofs.
This parameterization requires $n_f$ as a hyperparameter, the number of independent constraints among the flows.
For a connected electrical network with one element per branch, it is the number of nodes minus one, since the row space of $W_f$ is spanned by Kirchhoff's current law at the nodes~\citep[Secs.~2.2.2 and~12.1]{vanderschaft2014port}.
A constant $W$ does not require the constraints to be linear in the state: the force balance of several nonlinear springs at a massless junction is linear in the efforts and is inside this setting.

\paragraph{Learning and prediction.}
\label{sec:learning}
Given a set of the trajectories of inputs $u(t)$ and observations $y(t_k)$ at sample times $t_k$ for $k=0,1,\dots$, we train DINEs to fit them.
The observations $y$ are a given subset of the port variables and states; unlike Eq.~\eqref{eq:reduction}, they need not be the variables conjugate to the inputs.
We prepare a neural network $\chi_\eta$ called trajectory network with parameters $\eta$ that takes the time $t$ and the index of the trajectories and returns, as continuous-time functions of $t$, the state $z(t)$ of every storage element, the flow $f_R$ of every resistive element, and the conjugate variable of every source port, whether observed or not.
The remaining port variables follow from the constitutive relations.
For storage elements, $\dot{z}_i$ and $\partial H_i/\partial z_i$ are computed by the automatic differentiation, and the effort of a resistive element follows from $e_{R,j}=\hat{\gR}_{\theta,j}(f_{R,j})$.
The observations $\hat{y}(t)$ are reconstructed from these states and port variables.
We suppose that the number of elements of each type (C, I, R, Sf, Se) is given, which fixes the numbers of ports of each type and hence $n$.

We train this neural network by minimizing the sum of the data-assimilation loss at the sample times $t_k$ and the DAE residual of Eq.~\eqref{eq:dirac-dae} at collocation points $s_j$:
\begin{equation}\textstyle
    \lambda \mathbb{E}_k \big\| \hat{y}(t_k) - y(t_k) \big\|^2
    + \mathbb{E}_j \big\| F f(s_j) + E e(s_j) \big\|^2,
    \label{eq:loss}
\end{equation}
for a weight $\lambda$.
Each term is averaged over a minibatch of trajectories, times $t_k$, and collocation points $s_j$, which are randomly sampled at every update, and the parameters $\theta$ of $H_\theta$ and $\gR_\theta$, the matrix $\Omega$, and the parameters $\eta$ of $\chi_\eta$ are optimized jointly by gradient descent.
$\chi_\eta$ takes the role of a numerical integrator during training, solving the learned DAE by minimizing its residual, and is discarded after training.
This approach is similar to physics-informed neural networks (PINNs)~\citep{raissi2019physics}, except that the equation itself is learned.
For prediction, initial values are estimated from the observations available at the start of the prediction so that they are consistent with the algebraic rows of the learned DAE, and the learned DAE in Eq.~\eqref{eq:dirac-dae} is integrated by the three-stage Radau IIA method of order five~\citep{hairer1996solving} with Newton's method~\citep{nocedal2006numerical}.
See Appendix~\ref{app:implementation} for the details of implementation.

\paragraph{Isolation and composition.}
\label{sec:composition}
The learned DINEs contain the constitutive relations of the elements and the Dirac structure $(F, E)$ separately, so subsystems can be isolated or composed by operations on their ports, without eliminating variables, deriving reduced equations, or retraining, up to scale calibrations.
A subpart of the system can be isolated by removing the constitutive relations of the elements outside it, while keeping their ports in $(F, E)$ as boundary ports, i.e., ports that are not connected to any element.
When considering boundary ports to be connected to non-prescribed sources (Sf or Se), the isolated system is again DINEs.
In particular, when the subpart is a single element, its own port becomes the boundary port, and the element is isolated with its learned constitutive relation alone and can be reused in another system.
Two isolated systems $S_A, S_B$ can be composed by adding a boundary port to each of their Dirac structures and connecting these ports to each other.
This process stacks their matrices $(F_A, E_A)$ and $(F_B, E_B)$ block-diagonally and adds two rows that impose $f_A + f_B = 0$ and $e_A - e_B = 0$ on the connected port variables, with signs depending on the orientation of the ports.
These rows correspond to $\gD_{int}$ in Fig.~\ref{fig:concept}.
The composed DAE has $n_A + n_B + 2$ constraints in the $n_A + n_B$ ports, which can be integrated directly.
Alternatively, the connected ports are eliminated algorithmically before integration, after which the composed DAE has again a separable and constant Dirac structure on the remaining ports and hence a DINE~\citep[Sec.~6]{vanderschaft2014port}.

\paragraph{Discussions for comparisons and limitations.}
\label{sec:discussion}
Thanks to the kernel representation and the trajectory network $\chi_\eta$, DINEs do not require reduction and admit partial observation.
In contrast, PoDiNNs require the input-output representation to exist, possibly after reduction, and every state to be observed.
OE-pHNNs and the pH-DAE identification suppose the observation to be the output conjugate to the input.
Our parameterization of $(F, E)$ requires $n_f$, which has to be swept in practice because the interconnection is unknown.
Nevertheless, a wrong model is not identified silently, because a misspecification of $n_f$ leads to a catastrophic failure; this is expected from the implication of Theorem~\ref{thm:dirac_parameterize} that a misspecified $n_f$ excludes the true Dirac structure from the model family.

Gyrators, which typically bridge physical domains as in DC motors and hydraulic motors, are outside our setting.
Constraints whose coefficients depend on the state, which lead to modulated Dirac structures as for a multi-link robot in Cartesian coordinates, and a kinetic energy that depends on the momenta and the positions together, which violates the form $H_\theta(z) = \sum_i H_{\theta,i}(z_i)$ as for the same robot in generalized coordinates, are also outside our setting.

When reading the learned interconnection, it is beneficial to transform $W_f$ and $W_e$ to reduced row echelon forms, so that each row states one constraint among a minimal set of ports.
Depending on the interconnection and on which elements are observed, the scales of port variables, constitutive relations, and states are not unique.
This gauge freedom is removed in most cases by normalizing the entries of the rows corresponding to junctions (but not to transformers) to $0$ or $\pm 1$; otherwise, a few additional measurements fix the remaining scales, before or after isolation and composition.

\section{Experiments and results}
\label{sec:experiments}
\subsection{Experimental settings}
\label{sec:settings}
\begin{figure}[t]
    \centering
    \small
    \setlength{\tabcolsep}{3pt}
    \begin{tabular}{ccc}
        \includegraphics[height=1.8cm]{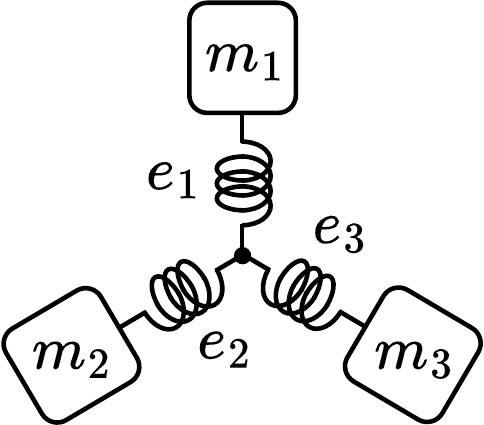} &
        \includegraphics[height=2.1cm]{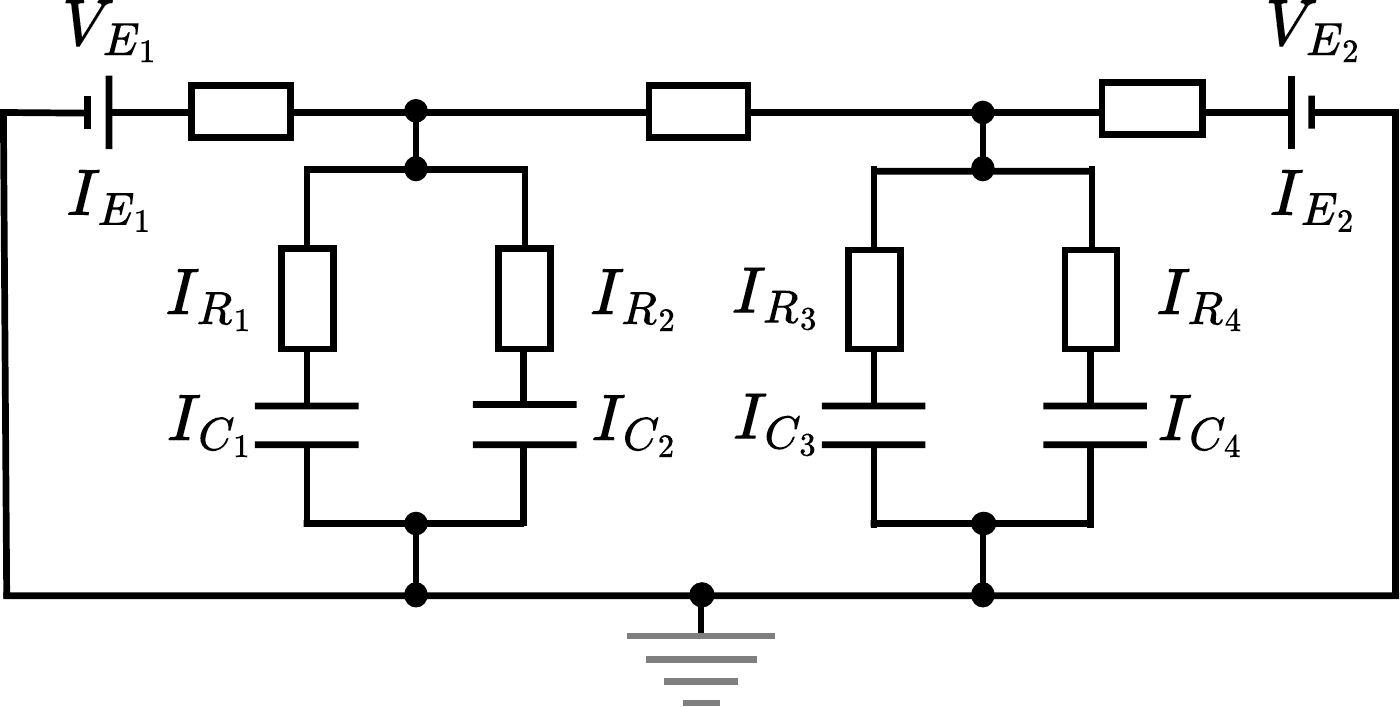}          &
        \includegraphics[height=1.4cm]{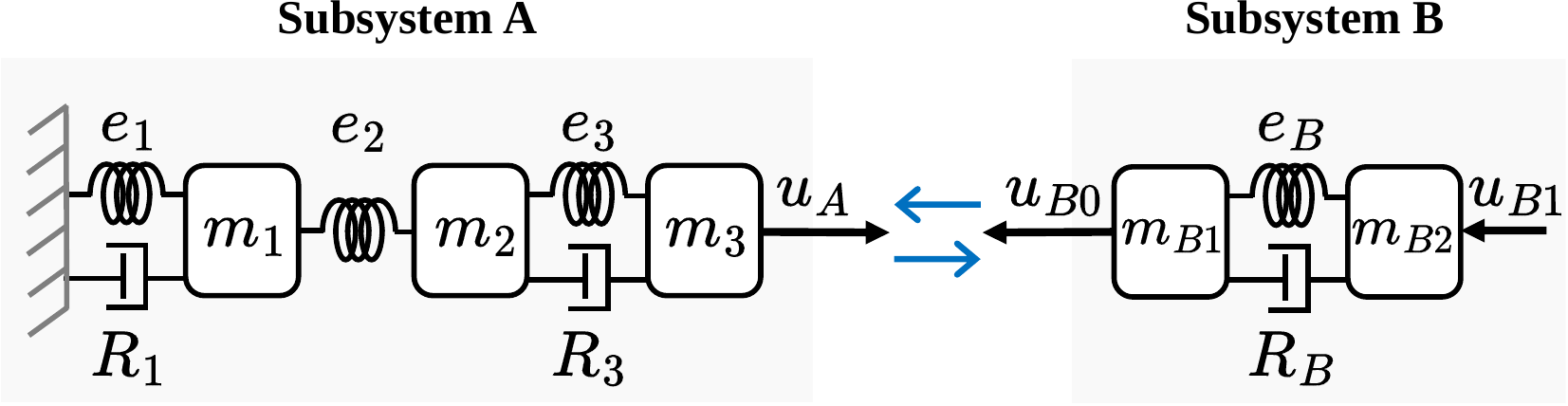}                \\
        \footnotesize (Task 1) Three-way junction                            &
        \footnotesize (Task 2) RC ladder                                     &
        \footnotesize (Task 3) Rigid coupling
    \end{tabular}\\[0.5em]
    \setlength{\tabcolsep}{6pt}
    \begin{tabular}{cc}
        \includegraphics[height=1.8cm]{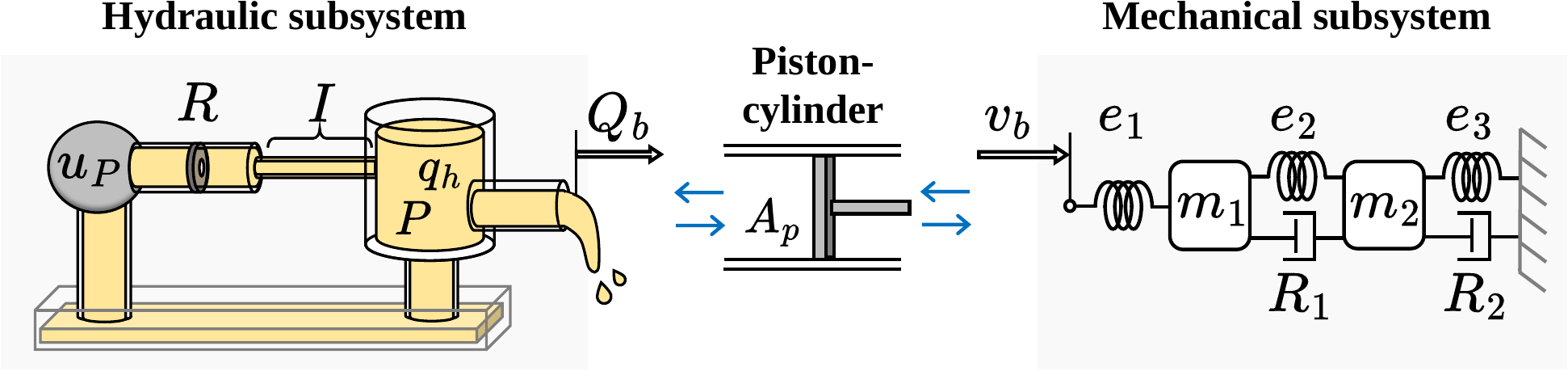} &
        \includegraphics[height=1.5cm]{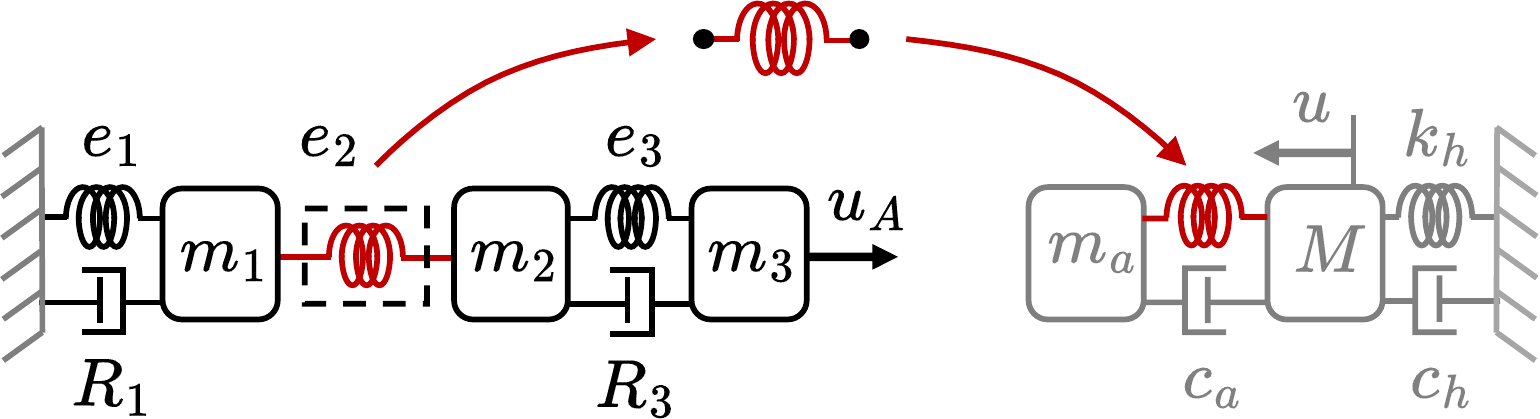}    \\
        \footnotesize (Task 4) Hydraulic--mechanical             &
        \footnotesize (Task 5) Spring reuse
    \end{tabular}
    \caption{Diagrams of the systems of Tasks 1--5.}
    \label{fig:systems}
\end{figure}

\textbf{Tasks.}
We evaluate DINEs on the five tasks shown in Figure~\ref{fig:systems}.
Appendix~\ref{app:experiments} gives their detailed settings, equations, parameters, and data generation.
In Task~1, a \emph{three-way junction}, we examine whether the constraints are learned.
The force balance at the junction leads to an algebraic constraint, which PoDiNNs cannot learn.
In Task~2, an \emph{RC ladder}, we examine whether the system can be learned from observed currents when the states (that is, capacitor charges) are unobserved.
Its node voltages are determined by nonlinear algebraic relations that arise from the constitutive relations of the resistors and Kirchhoff's laws.
In Tasks~3 and~4, we examine the composition of learned subsystems, and in Task~5 the isolation.
Each subsystem has one boundary port, which is connected to a source during learning, and then composed to another.
In Task~3, a \emph{rigid coupling}, the sources are external forces, applied to two masses.
In Task~4, a \emph{hydraulic--mechanical coupling}, the boundary inputs are the flow leaving the reservoir and the velocity imposed at the free end of the first spring.
Because of the multi-domain nature of the system, this composition is represented by a transformer, whose constant ratio (that is, the piston area) is unknown and identified from a few measurements of the composed system.
In Task~5, a \emph{spring reuse}, we examine whether the spring learned in a subsystem (i.e., a nonlinear spring) of Task 3 can be isolated and reused in a new mechanical system without retraining.

\textbf{Baselines and prior knowledge.}
None of the existing methods introduced in Section~\ref{sec:related} is applicable to Tasks~1--5 as it is, except Neural ODEs to Task~1.
On Task~1, the other methods reduce to Neural ODEs, require the constraints to be given in advance, or derive the state-dependent coefficient matrices.
We nevertheless evaluate PoDiNNs on Task~1 to confirm this limitation.
For Task~2, we extend the observation map of CT-SUBNET from $y = h_\theta(z)$ to $y = h_\theta(z, u)$ because the observed currents $y$ depend directly on the source voltages $u$.
CT-SUBNET is given its state dimension $n_z$, which is swept.
Recall that OE-pHNNs and the pH-DAE identification assume the observation to be conjugate to the input, and N-pHDAEs require the circuit topology to be known and all of its variables observable.
For Tasks~3 and~4, we construct Coupled Neural ODEs as a baseline without physical structure, where a Neural ODE learns each subsystem and a separate network learns the variable conjugate to the boundary input, and then these networks are composed.
For Task~4, the piston area is identified from a few observations of the composed system.
For DINEs, we assume that the I-type elements, the capacitors of Task~2, and the reservoir of Task~4 are linear with their coefficients learned.
For comparison, we also evaluate DINEs (soft), which replace the parameterization of Theorem~\ref{thm:dirac_parameterize} by penalty terms of the conditions of Eq.~\eqref{eq:conditions} in the loss (see Appendix~\ref{app:ablation} for details).
We summarize the prior knowledge and assumptions available for each model in Appendix~\ref{app:additional}.

\textbf{Training and prediction.}
DINEs are trained as stated above, whereas Neural ODEs, PoDiNNs, and the subsystem models of Coupled Neural ODEs are trained by one-step prediction with the classical fourth-order Runge--Kutta method, and CT-SUBNET by the multi-step prediction error over windows whose initial state is given by its encoder.
The learned DAEs are integrated by the Radau IIA method and the ODE models by an adaptive Runge--Kutta method.
On Tasks~3 and~4, the learned subsystems are composed without retraining: the scale of the boundary variable of each subsystem of DINEs is calibrated once from the learned interconnection or from boundary measurements in its training data (see Appendix~\ref{app:coupling}), the piston area of Task~4 is identified from a few states of the composed system, and the composed system is integrated as a DAE whose boundary variables are unknowns.

For prediction, DINEs estimate the initial state to be consistent with the observation at the first sample (or the history of the inputs and observations on Task~2) under the learned algebraic constraints.
Neural ODEs and PoDiNNs use the observed initial state.
CT-SUBNET estimates the initial state from the history using its encoder.
Coupled Neural ODEs do the same as DINEs.

\textbf{Evaluation metrics.}
Each model predicts the trajectory of $K$ steps of every test trajectory, and the squared error between the predicted and true observations, averaged over the $d$ observed channels, $e_k = \frac{1}{d} \| \hat{y}(t_k) - y(t_k) \|^2$, is computed at every sample $k = 0, \dots, K$.
The MSE is its average over $k$.
The valid prediction time (VPT)~\citep{vlachas2020backpropagation} is the duration until the step at which the error first exceeds a threshold $\varepsilon$, normalized by the length of the test trajectory.
Both metrics are averaged over the ten test trajectories for each seed, and we report the mean and the standard deviation over the five seeds.
See Appendix~\ref{app:common} for the detailed definitions.

\subsection{Results}
\label{sec:results}
\begin{table}[t]
    \centering
    \caption{Experimental results.}
    \label{tab:results}
    \scriptsize
    \setlength{\tabcolsep}{3pt}
    \begin{tabular}[t]{l cc}
        \toprule
                                    & \multicolumn{2}{c}{\textbf{(Task 1) Three-way junction}}                                                     \\
        \cmidrule(lr){2-3}
        \textbf{Model}              & \textbf{MSE}$\downarrow$                                                        & \textbf{VPT}$\uparrow$     \\
        \midrule
        Neural ODEs                 & $(2.871\sd{1.014}) \times 10^{4}$                                               & $0.062\sd{0.015}$          \\
        PoDiNNs$^*$                 & $(6.129\sd{0.089}) \times 10^{4}$                                               & $0.000\sd{0.000}$          \\
        \midrule
        DINEs (soft)$^\dagger$      & \multicolumn{2}{c}{(not converged in 3 of 5 seeds)}       \\
        DINEs (ours)                & $\mathbf{3.319}\sd{1.221}$                                                      & $\mathbf{0.860}\sd{0.043}$ \\
        \midrule
        DINEs ($n_f = 1$)$^\ddagger$ & \multicolumn{2}{c}{(not converged in all seeds)}                                                                          \\
        DINEs ($n_f = 3$)$^\ddagger$ & $(4.975\sd{0.062}) \times 10^{4}$                                               & $0.000\sd{0.000}$          \\
        DINEs ($n_f = 4$)$^\ddagger$ & $(6.697\sd{0.912}) \times 10^{4}$                                               & $0.000\sd{0.000}$          \\
        DINEs ($n_f = 5$)$^\ddagger$ & $(6.347\sd{0.058}) \times 10^{4}$                                               & $0.000\sd{0.000}$          \\
        \bottomrule\addlinespace[1pt]
                                    & $\times 10^{-5}$                                                                & $\varepsilon = 10^{-4}$    \\
    \end{tabular}\qquad
    \begin{tabular}[t]{l cc}
        \toprule
                               & \multicolumn{2}{c}{\textbf{(Task 2) RC ladder}}                                              \\
        \cmidrule(lr){2-3}
        \textbf{Model}         & \textbf{MSE}$\downarrow$                                        & \textbf{VPT}$\uparrow$     \\
        \midrule
        CT-SUBNET ($n_z = 4$)  & $11.248\sd{2.962}$                                              & $0.039\sd{0.011}$          \\
        CT-SUBNET ($n_z = 8$)  & $3.300\sd{0.899}$                                               & $0.096\sd{0.084}$          \\
        CT-SUBNET ($n_z = 12$) & $3.267\sd{0.869}$                                               & $0.157\sd{0.023}$          \\
        CT-SUBNET ($n_z = 16$) & $3.683\sd{1.077}$                                               & $0.205\sd{0.055}$          \\
        CT-SUBNET ($n_z = 20$) & $3.864\sd{1.846}$                                               & $0.241\sd{0.018}$          \\
        \midrule
        DINEs (soft)           & \multicolumn{2}{c}{(not converged in all seeds)}         \\
        DINEs (ours)           & $\mathbf{1.003}\sd{0.113}$                                      & $\mathbf{0.711}\sd{0.002}$ \\
        \bottomrule\addlinespace[1pt]
                               & $\times 10^{-5}$                                                & $\varepsilon = 10^{-4}$    \\
    \end{tabular}\\[1ex]
    \begin{tabular}{l cc cc cc}
        \toprule
                            & \multicolumn{2}{c}{\textbf{(Task 3) Rigid coupling}}   & \multicolumn{2}{c}{\textbf{(Task 4) Hydraulic--mechanical}} & \multicolumn{2}{c}{\textbf{(Task 5) Spring reuse}}                                                                                                                                                              \\
        \cmidrule(lr){2-3} \cmidrule(lr){4-5} \cmidrule(lr){6-7}
        \textbf{Model}      & \textbf{MSE}$\downarrow$                               & \textbf{VPT}$\uparrow$                                      & \textbf{MSE}$\downarrow$                               & \textbf{VPT}$\uparrow$                                 & \textbf{MSE}$\downarrow$                      & \textbf{VPT}$\uparrow$                        \\
        \midrule
        Coupled Neural ODEs & $189.508\sd{101.995}$                                  & $0.002\sd{0.001}$                                           & $73.893\sd{10.636}$                                    & $0.004\sd{0.001}$                                                                                                                                      \\
        \midrule
        DINEs (ours)        & $\mathbf{1.206}\sd{0.492}$                             & $\mathbf{0.970}\sd{0.068}$                                  & $\mathbf{1.068}\sd{0.837}$                             & $\mathbf{0.831}\sd{0.321}$                             & $6.963\sd{4.625}$                             & $0.968\sd{0.050}$ \\
        \bottomrule\addlinespace[1pt]
                            & $\times 10^{-7}$                                       & $\varepsilon = 10^{-6}$                                     & $\times 10^{-7}$                                       & $\varepsilon = 10^{-6}$                                & $\times 10^{-8}$                              & $\varepsilon = 10^{-6}$                       \\
    \end{tabular}    \\[0.5ex]
    \raggedright\scriptsize Means $\pm$ standard deviations over five trials.
    The scale of the MSE and the threshold $\varepsilon$ of the VPT are given in the last row of each block.
    $^*$~For reference only, since PoDiNNs are theoretically inapplicable to Task~1.
    $^\dagger$~3 of 5 seeds did not converge.
    $^\ddagger$~$n_f$ is misspecified.

\end{table}

\begin{figure}[t]
    \centering
    \small
    \begin{subfigure}[t]{0.19\linewidth}
        \centering
        \includegraphics[width=\linewidth]{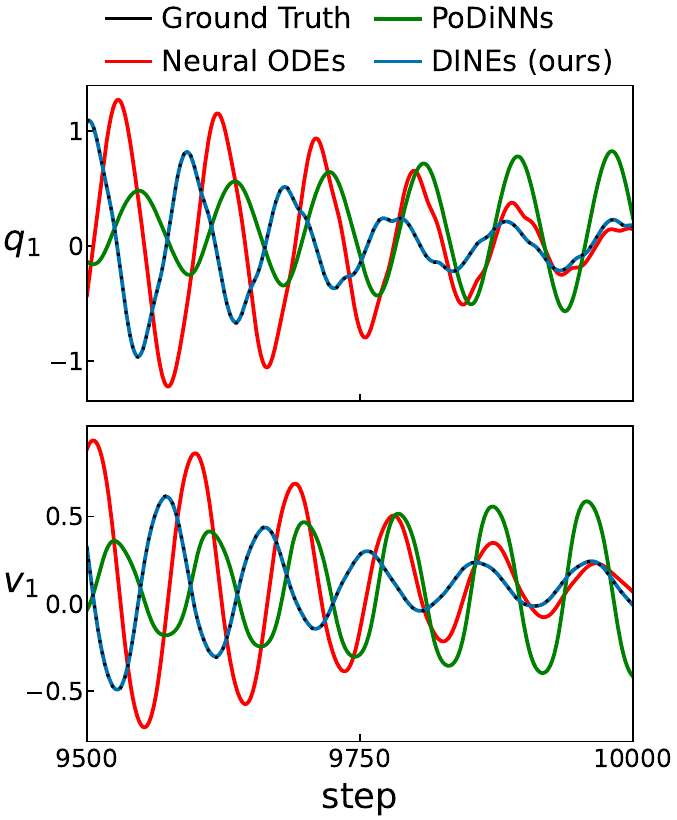}
        (Task~1) three-way junction.
        \label{fig:results:three-way}
    \end{subfigure}%
    \begin{subfigure}[t]{0.405\linewidth}
        \centering
        \includegraphics[width=\linewidth]{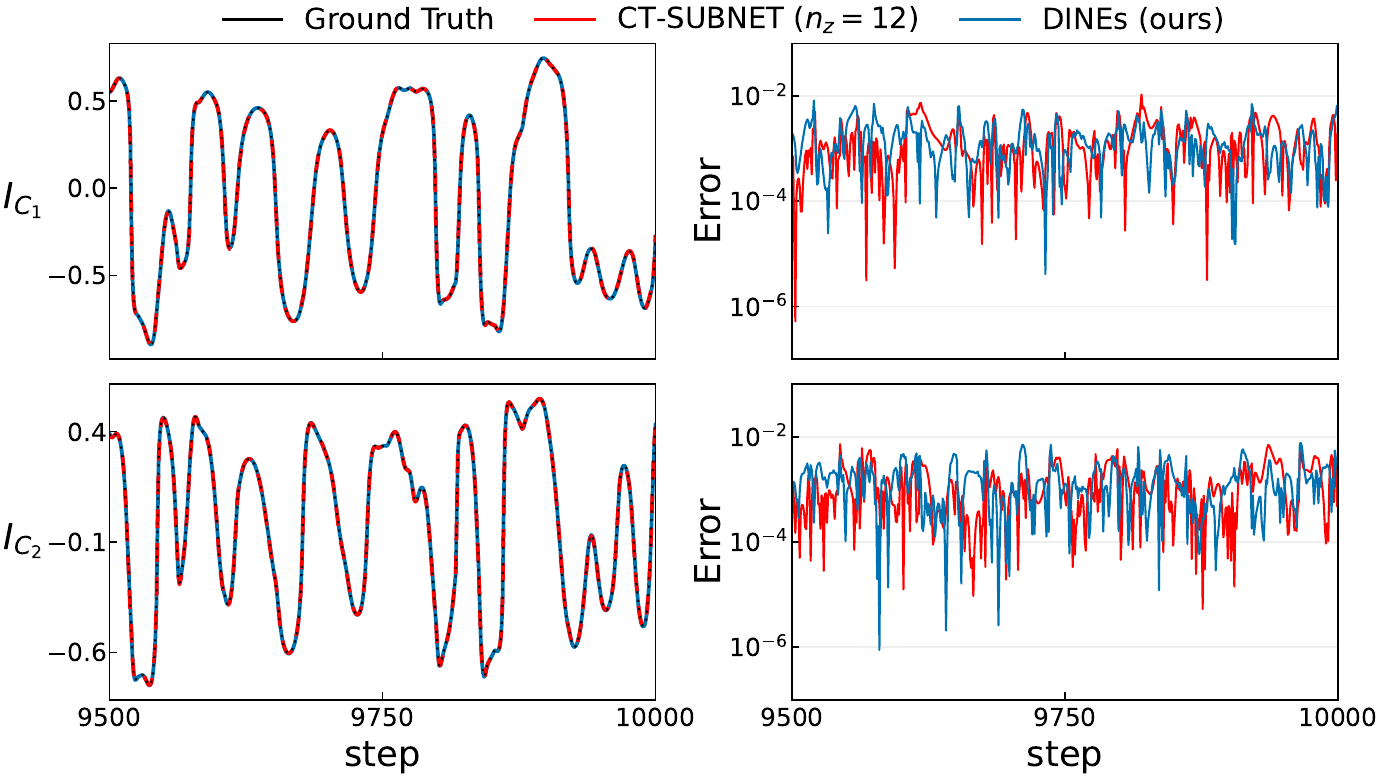}
        (Task~2) RC ladder ($n_z=12$).
        \label{fig:results:rc12}
    \end{subfigure}%
    \begin{subfigure}[t]{0.405\linewidth}
        \centering
        \includegraphics[width=\linewidth]{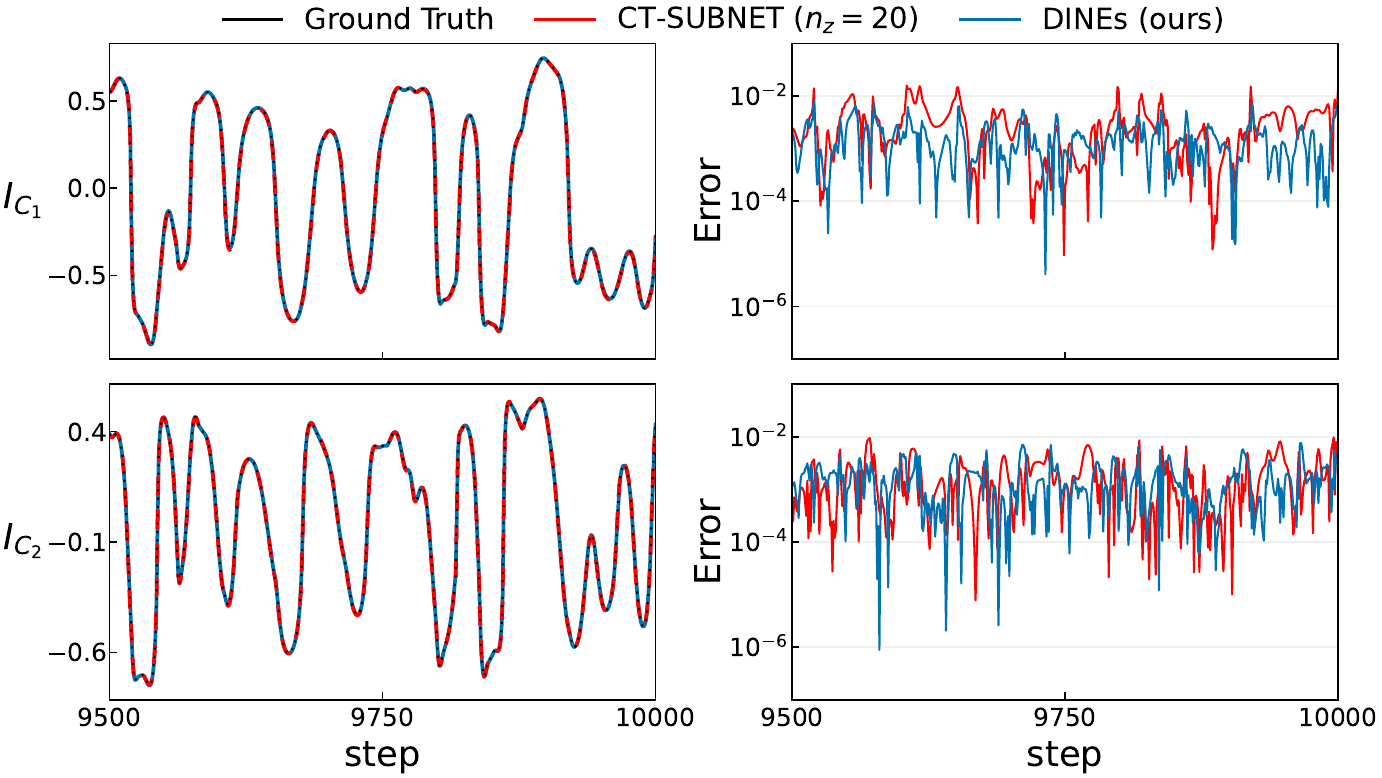}
        (Task~2) RC ladder ($n_z=20$).
        \label{fig:results:rc20}
    \end{subfigure}\\
    \begin{subfigure}[t]{0.406\linewidth}
        \centering
        \includegraphics[width=\linewidth]{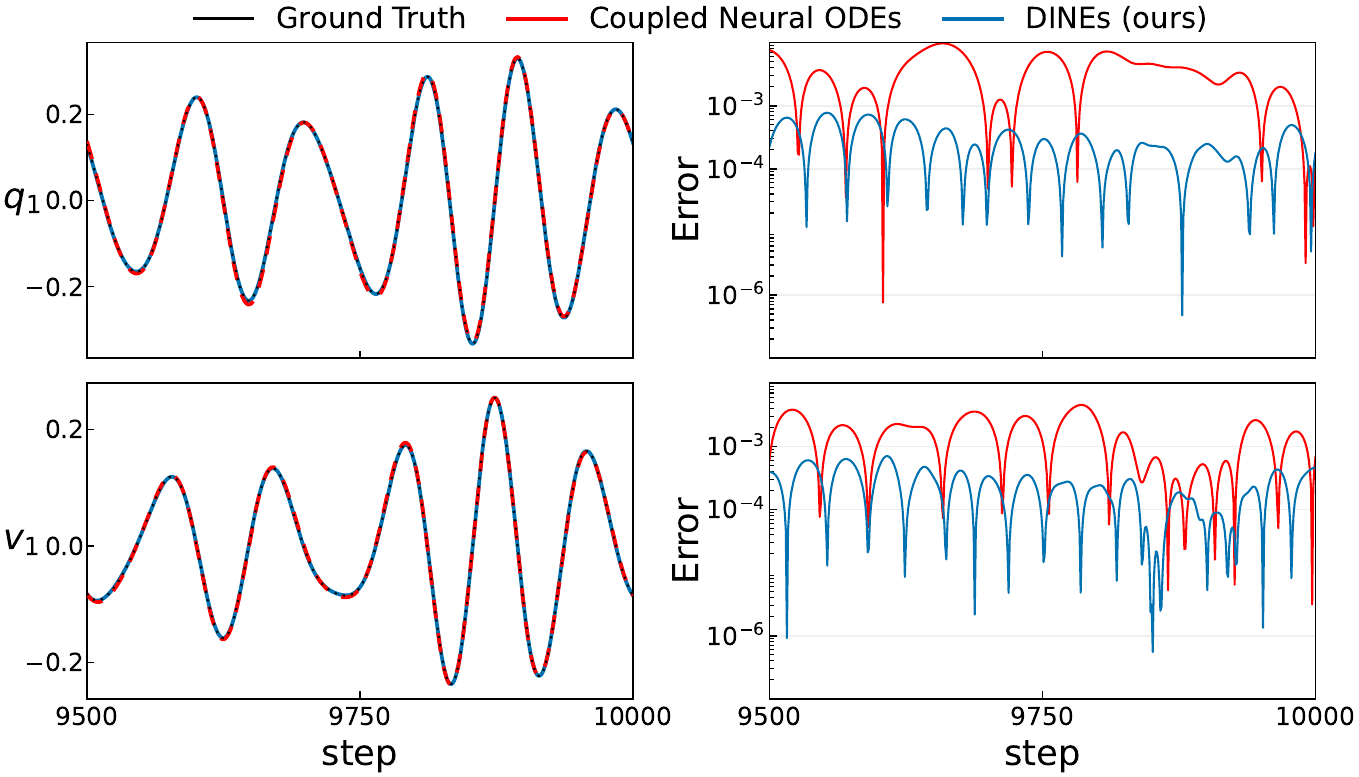}
        (Task~3) rigid coupling.
        \label{fig:results:rigid}
    \end{subfigure}%
    \begin{subfigure}[t]{0.394\linewidth}
        \centering
        \includegraphics[width=\linewidth]{result/piston.pdf}
        (Task~4) hydraulic--mechanical coupling.
        \label{fig:results:piston}
    \end{subfigure}%
    \begin{subfigure}[t]{0.20\linewidth}
        \centering
        \includegraphics[width=\linewidth]{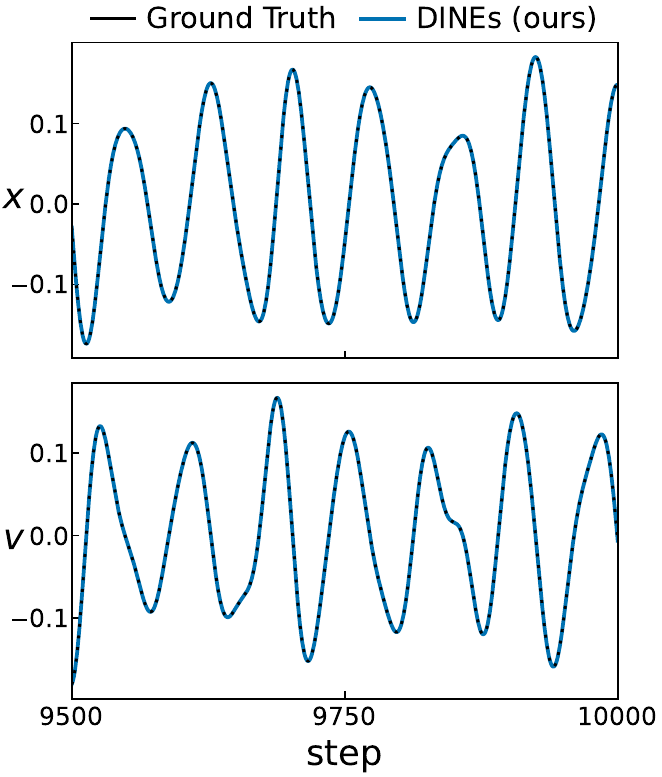}
        (Task~5) reuse.
        \label{fig:results:reuse}
    \end{subfigure}
    \caption{Visualizations of example predictions. Error denotes the absolute error.}
    \label{fig:results}
\end{figure}

\textbf{Results and visualizations.}
Table~\ref{tab:results} summarizes the numerical performances, and Figure~\ref{fig:results} visualizes example predictions.
DINEs achieved the lowest MSEs and the highest VPTs in all Tasks~1--4, and also showed high prediction accuracy on Task~5.

Neural ODEs can represent the system in Task 1 as a reduced ODE in principle, but their prediction errors grew quickly.
This is because the training trajectories lie on the manifold defined by the constraints of the interconnection, and a prediction becomes out of distribution once it leaves the manifold.
DINEs learn the interconnection and thus suppress such errors.
PoDiNNs cannot learn this kind of constraint and also perform poorly.
A variant with a state-dependent matrix $B_\theta(q)$ is not evaluated here as not examined in the original paper.
On Task~2, we compared CT-SUBNET with the latent state dimension $n_z$ ranging from the true number $4$ to a redundant $20$.
The MSE improved from $n_z = 4$ to $n_z = 8$ and deteriorated slightly beyond $n_z = 12$, whereas the VPT kept improving up to $n_z = 20$ with decreasing gains.
None of these settings reaches the level of DINEs.

On Tasks~3 and~4, the predictions of Coupled Neural ODEs look reasonable in Figure~\ref{fig:results}, but DINEs achieved MSEs nearly two orders of magnitude smaller and much longer VPTs.
Coupled Neural ODEs can be composed only when they learn the relation between the boundary port variables given as inputs and their conjugate variables as outputs.
In contrast, DINEs have an advantage that the Dirac structure provides such relation up to the gauge freedom.

Task~5 has no baseline because DINEs are the only method applicable to it.
DINEs achieved the MSE and VPT comparable to (or better than) those on the other tasks.
Although this is partly because every component other than the spring reused from Task~3 is known in this setting, the result shows that the isolation and the composition were performed with high accuracy.

\textbf{Necessity of hard constraints.}
For DINEs (soft), which impose only soft constraints, the Newton method for initialization or prediction did not converge in most seeds on Tasks~1 and~2.
This is presumably because a solution of the DAE may not exist when the constraints are only soft.
This result shows the necessity of the parameterization derived in Theorem~\ref{thm:dirac_parameterize}.

\textbf{Impact of the number of constraints.}
DINEs require the number $n_f$ of flow constraints, which is unknown in general.
To investigate the impact of its misspecification, we examine $n_f = 1, \ldots, 5$ on Task~1.
With $n_f = 1$, the Newton method for prediction did not converge.
Note that this problem does not arise in training, where the trajectory network is used.
With $n_f = 3, 4, 5$, the Newton method converged, but the MSEs and VPTs were catastrophic.
Since the number of effort constraints is $n_e = n - n_f$, $n_f$ determines not the total number of constraints but their allocation between flows and efforts.
Because the MSE and VPT differ so clearly, $n_f$ can be tuned easily by holding out a small validation set from the training data, and a wrong model is not identified silently.

\textbf{Learned interconnection.}
We confirmed that the matrices $F$ and $E$ learned the interconnection correctly in Tasks 1--4.
Note that, since the Dirac structure depends only on the row spaces of $F$ and $E$, the learned rows are not unique, and a readout procedure is needed for interpretation.
See Appendix~\ref{app:additional} for examples of the readout of the learned interconnection.

\section{Conclusion}
\label{sec:conclusion}
We proposed DINEs, which learn a physical system as a DAE whose algebraic constraints are given by a separable Dirac structure parameterized exactly by an orthogonal matrix, without reducing the system to an ODE.
Our experiments showed that DINEs learn constraints and unobserved states that existing methods cannot, and that learned subsystems and elements are composed and reused across physical domains without retraining.
The automatic determination of $n_f$ and the extension to non-separable Dirac structures are left for future work.

\subsection*{AI use statement}
We used large language models to translate, edit, and proofread the manuscript text and to check LaTeX formatting and bibliographic details.
We also used these models to survey related work.
All AI-assisted text was reviewed and revised by the authors, and all derivations were checked manually by the authors.
We take responsibility for the final content of this work, including text, claims, and artifacts produced with the aid of generative AI.

\subsection*{Ethics statement}
This work concerns the data-driven modeling of physical systems and uses only simulated data generated from known equations of motion.
It involves no human subjects, personal data, or data releases with privacy implications, and we do not foresee direct harmful applications.

\subsection*{Reproducibility statement}
The construction of the interconnection matrices $F$ and $E$ and the training loss, Eq.~\eqref{eq:loss}, are specified in Section~\ref{sec:method}, and the derivation of the reduced form used for comparison is given in Appendix~\ref{app:reduction}.
Appendix~\ref{app:implementation} provides the implementation of DINEs and its ablated variants, and Appendix~\ref{app:experiments} shows the equations of motion, parameters, observed variables, sampling settings, and optimizer settings of Tasks 1--5, together with the architectures and training of the baselines.
Source code and the data-generation scripts will be made public after refactoring.

\bibliography{arxiv}
\bibliographystyle{arxiv}

\clearpage
\appendix
\section{Elements and junction structure of port-based modeling}
\label{app:domains}
This appendix supplements Section~\ref{sec:port}.
Table~\ref{tab:domains} summarizes the power-conjugate variables, the state variables, and the elements of port-based modeling in the mechanical, electromagnetic, and hydraulic domains.
For example, the resistive relation of a linear damper or resistor is $e_R = -\hat{R} f_R$ with a positive coefficient $\hat{R}$.
A source (Sf or Se) supplies energy from, or drains it to, the outside of the system by prescribing a flow (e.g., a current source) or an effort (e.g., a voltage source or an external force) as a function of time, independently of the conjugate variable.
In this paper, each of the storage elements, resistive elements, and sources is a one-port, that is, it has a single pair of port variables $(f, e)$, and energy is stored, dissipated, or supplied only inside these elements.
A storage element whose energy depends on several states, such as the inertia of a multi-link robot whose kinetic energy depends on the momenta and the positions together, is a multiport and is outside our setting.
The junction structure is an assemblage of the following four types of elements~\citep{karnopp2012system}.
A 0-junction imposes a common effort with balanced flows, and a 1-junction imposes a common flow with balanced efforts, which generalize Kirchhoff's current and voltage laws, respectively.
In a mechanical system, a spring between two masses applies a common force to both and is connected to them by a 0-junction, whereas the springs and dampers that connect a mass to a fixed wall have a common velocity and are connected by a 1-junction.
A transformer (TF) is a two-port element that relates the efforts of its two ports, and likewise the flows, by a common ratio $n$, that is, $f_2 = n f_1$ and $e_1 = -n e_2$, as in a lever, a gear box, or a piston--cylinder pair.
A gyrator (GY) is a two-port element that relates the effort of each port to the flow of the other by a common ratio $r$, that is, $e_1 = -r f_2$ and $e_2 = r f_1$, as in a DC motor pump.
The two-ports are power continuous as $e_1 f_1 + e_2 f_2 = 0$ holds, where the power at both ports is counted as flowing into the element.
\begin{table}[h]
    \centering
    \caption{Power-conjugate variables, state variables, and elements in each domain}
    \label{tab:domains}
    \small
    \setlength{\tabcolsep}{4pt}
    \begin{tabular}{l l l l l}
        \toprule
            & Bond graph              & Mechanical      & Electromagnetic & Hydraulic          \\
        \midrule
        $f$ & flow                    & velocity        & current         & volume flow        \\
        $e$ & effort                  & force           & voltage         & pressure           \\
        $q$ & time integral of flow   & displacement    & charge          & volume             \\
        $p$ & time integral of effort & momentum        & flux linkage    & flow tube momentum \\
        \midrule
        C   & C-type storage element  & spring          & capacitor       & reservoir          \\
        I   & I-type storage element  & mass            & inductor        & fluid inertia      \\
        R   & resistive element       & damper          & resistor        & fluid resistor     \\
        Sf  & flow source             & moving boundary & current source  & flow rate source   \\
        Se  & effort source           & external force  & voltage source  & pressure source    \\
        \bottomrule
    \end{tabular}\\
    \raggedright Note that the flow or the effort of a port can be the negative of the listed quantity as $\langle e \mid f \rangle$ is the power flowing out of the element.
    For example, the effort of a mass is the negative of the force acting on it.
\end{table}

\section{Reduction of \texorpdfstring{Eq.~\eqref{eq:dirac-dae}}{the Dirac DAE} to the input-state-output form}
\label{app:reduction}
This appendix supplements Section~\ref{sec:background}: it derives the input-output representation and Eq.~\eqref{eq:reduction} from Eq.~\eqref{eq:dirac-dae}, following \citet[Secs.~4 and 5]{vanderschaft2014port}, gives the correspondence of existing methods to Eq.~\eqref{eq:reduction}, and states the two costs of reduction in equations.

\textbf{Exchange of port variables.}
Exchanging the flow and the effort of a port exchanges the corresponding columns of $F$ and $E$, which preserves the condition $E F^\top + F E^\top = 0$ and the rank condition~\citep[Sec.~5.3]{vanderschaft2014port}.
We exchange them for every I-type port, every Sf port, and a subset of the resistive ports, and denote the resulting matrices, flows, and efforts with $\tilde{ }$.
Every storage port then satisfies $\tilde{f}_S = -\dot{z}$ and $\tilde{e}_S = \nabla H(z)$, and every source port satisfies $\tilde{e}_P = u$, where $\nabla H(z) \in \R^{n_S}$ is the vector of the partial derivatives $\partial H / \partial z_i$.
Moreover, $\tilde{E} \tilde{F}^\top + \tilde{F} \tilde{E}^\top = 0$ and $\operatorname{rank} [\, \tilde{F} \mid \tilde{E} \,] = n$ hold.

\textbf{Input-output representation.}
If $\tilde{F}$ is invertible for the chosen subset of the resistive ports, Eq.~\eqref{eq:kernel} is equivalent to $\tilde{f} = B \tilde{e}$ with $B := -\tilde{F}^{-1} \tilde{E}$, that is, $\gD$ is the graph of $B$ in these variables.
This is the special case without constraints of the constrained input-output representation of \citet{dalsmo1998representations} and \citet[Sec.~5.2]{vanderschaft2014port}.
Here, $B$ is skew-symmetric because multiplying $\tilde{E} \tilde{F}^\top + \tilde{F} \tilde{E}^\top = 0$ by $\tilde{F}^{-1}$ from the left and by $\tilde{F}^{-\top}$ from the right gives $B + B^\top = 0$.
Since $\tilde{f}_S = -\dot{z}$, Eq.~\eqref{eq:dirac-dae} becomes an ODE in $z$ once the remaining port variables are eliminated through the constitutive relations, which is possible, for example, when $B_{RR} = 0$ and the resistive relations are explicit maps in these variables; this is the form that PoDiNNs employ~\citep{podinn}.

\textbf{Input-state-output form.}
We can decompose $B$ into submatrices corresponding to the storage, resistive, and source ports, and write $J := -B_{SS}$, $G_R := -B_{SR}$, and $G := -B_{SP}$.
Suppose, as is typical, that $B_{RR}$, $B_{RP}$, and $B_{PP}$ vanish, that is, no resistive or source port is connected directly to another resistive or source port.
Also, suppose that the resistive relation is linear, $\tilde{e}_R = -\hat{R} \tilde{f}_R$ with $\hat{R} = \hat{R}^\top \succeq 0$, where the coefficient of an exchanged resistive port is the inverse of that in the original variables.
The storage rows of $\tilde{f} = B \tilde{e}$ give $\dot{z} = -\tilde{f}_S = J\, \nabla H(z) + G_R \tilde{e}_R + G u$, the resistive rows give $\tilde{f}_R = B_{RS}\, \nabla H(z) = G_R^\top\, \nabla H(z)$ by the skew-symmetry of $B$, and the source rows give $\tilde{f}_P = B_{PS}\, \nabla H(z) = G^\top\, \nabla H(z)$.
Substituting $\tilde{e}_R = -\hat{R} \tilde{f}_R = -\hat{R} G_R^\top\, \nabla H(z)$ into the storage rows and setting $y := \tilde{f}_P$ yields Eq.~\eqref{eq:reduction} with $R = G_R \hat{R} G_R^\top$.
If a resistive or source port were connected directly to another resistive or source port, $B_{RP}$ or $B_{PP}$ would not vanish, and the output $y$ would acquire a feedthrough term in $u$~\citep[Sec.~4.2]{vanderschaft2014port}.

\textbf{Existing methods as special cases of Eq.~\eqref{eq:reduction}.}
HNNs~\citep{greydanus2019hamiltonian} use $\dot{z} = J_c\, \nabla H_\theta(z)$ with a pair of positions and momenta as the state $z = (q, p) \in \R^{2d}$ and the standard symplectic matrix $J_c=\left[\begin{smallmatrix} 0 & I_d \\ -I_d & 0 \end{smallmatrix}\right]$, that is, Eq.~\eqref{eq:reduction} with $R = 0$, $G = 0$, and $J=J_c$; due to the skew-symmetry of $J_c$, they conserve the energy $H_\theta$ exactly along the exact solutions.
Neural symplectic forms~\citep{chen2021neural} further learn $J=J_\theta(z)$ as the inverse of a learned symplectic form.
SymODENs~\citep{zhong2020symplectic} learn $H_\theta$ and $G_\theta(q)$ with $J = J_c$ and $R = 0$, and Dissipative SymODENs~\citep{zhong2020dissipative} add a learned dissipation matrix $R_\theta(q) = \Gamma_\theta(q) \Gamma_\theta(q)^\top \succeq 0$.
pHNNs~\citep{desai2021port} and pseudo-Hamiltonian neural networks~\citep{eidnes2023pseudo} employ a predetermined $J$ and learn $H_\theta$ and $R_\theta$ (typically, its diagonal elements corresponding to the damping coefficients of the resistive elements), and learn an external force in place of $G u$, which depends only on time or state.
OE-pHNNs~\citep{moradi2026port} learn $J_\theta$, $R_\theta$, $G_\theta$, and $H_\theta$ from noisy observations $y$.
Lagrangian neural networks (LNNs) and deep Lagrangian networks (DeLaN)~\citep{lutter2019deep,cranmer2020lagrangian} are based on Lagrangian mechanics, but can be considered as Hamiltonian systems with tangent bundles, where $H$ and $J$ are derived from the learned Lagrangian.
CHNNs and CLNNs~\citep{finzi2020simplifying} add holonomic constraints given in advance by projecting the vector field onto the constraint manifold.

\textbf{Further details of related methods.}
N-pHDAEs~\citep{nphdae} rewrite the DAE in semi-explicit form by a QR decomposition of the matrices, reduce it to an ODE by index reduction, and integrate it with a Runge--Kutta method, whereas the pH-DAE identification~\citep{phdae} avoids reduction and solves the DAE by the backward Euler method; the latter suggests that the coefficient matrices can be learned but learns only the parameters of the elements in the experiments.
Latent ODEs~\citep{rubanova2019latent} encode the observed window into a distribution over the initial latent state $z_0$, predict the future trajectory, and decode each observation through a learned likelihood of $y$ given $z$.

\textbf{Cost 1: the interconnection is mixed with the elements.}
In Eq.~\eqref{eq:reduction}, $J$ and $G$ are blocks of the interconnection, whereas $R = G_R \hat{R} G_R^\top$ combines the constitutive relation $\hat{R}$ of the resistive elements with the blocks $G_R$ that connect them; a learned $R$ therefore cannot be decomposed uniquely into $\hat{R}$ and $G_R$.
For a nonlinear explicit relation $\tilde{e}_R = -\gR(\tilde{f}_R)$, the same substitution gives $\dot{z} = J\, \nabla H(z) - G_R\, \gR\big(G_R^\top\, \nabla H(z)\big) + G u$, so that the dissipative term depends on the state and no longer factors as a constant matrix times $\nabla H(z)$.

\textbf{Cost 2: an ODE is assumed to exist.}
When no choice of the exchanged ports makes $\tilde{F}$ invertible, e.g., for the balanced forces $\sum_i \partial H / \partial q_i = 0$ of multiple nonlinear springs connected to a massless point, an ODE in $z$ exists only after part of the state has been eliminated manually or $J$ is replaced with a rank-deficient matrix that is no longer determined by the interconnection alone but depends on the state and on the constitutive relations of the elements (e.g., the spring coefficients).
Task~1 in Section~\ref{sec:experiments} and in Appendix~\ref{app:task1} is an instance.
Likewise, for a capacitor connected in parallel with a voltage source, the effort constraint $\partial H / \partial q = u$ has no flow part, so that $\tilde{F}$ is singular and the state $q$ is determined by the input instead of by an ODE.

\section{Proofs}
\label{app:proofs}
\begin{proof}[Proof of Theorem~\ref{thm:dirac_parameterize}]
    The theorem states that a subspace $\gD \subset \gF \times \gE$ is a separable Dirac structure with exactly $n_f$ independent flow constraints if and only if $\gD = \gD(\Omega)$ for some $\Omega \in \R^{n_f \times n_e}$; part (i) is the sufficiency of this representation and part (ii) its necessity.

    \smallskip\noindent
    \textbf{Sufficiency (i).}
    Since $S(\Omega)^\top = -S(\Omega)$, $W(\Omega) = \exp\big(S(\Omega)\big)\, W_0$ is orthogonal.
    Hence, the rows of $W_f$ and $W_e$ are mutually orthogonal, so $W_eW_f^\top=0$.
    Moreover, as the rows within each block are orthonormal, $W_f$ and $W_e$ have full row rank.
    Therefore, $EF^\top + FE^\top = 0$ and $\operatorname{rank}[\,F\mid E\,] = n$.
    The flow rows and the effort rows are separate by construction, so $\gD(\Omega)$ is a separable Dirac structure with exactly $n_f$ independent flow constraints and $n_e$ independent effort constraints.

    \smallskip\noindent
    \textbf{Necessity (ii).}
    Let $\gD_\star$ be any separable Dirac structure with exactly $n_f$ independent flow constraints.
    By the definition of separability, $\gD_\star$ has a kernel representation, Eq.~\eqref{eq:kernel}, in which every row involves only flows or only efforts.
    Collecting the $n_f$ flow rows and the remaining $n_e$ effort rows, we can write its matrices as
    \begin{equation}
        F = \begin{bmatrix} W_f \\ 0 \end{bmatrix}, \qquad
        E = \begin{bmatrix} 0 \\ W_e \end{bmatrix}, \qquad
        W_f \in \R^{n_f \times n}, \quad W_e \in \R^{n_e \times n},
    \end{equation}
    up to the recombination of rows within each block, and the two conditions of Eq.~\eqref{eq:conditions} reduce to $W_e W_f^\top = 0$ and to $W_f$ and $W_e$ having full row rank, with $n_f + n_e = n$.
    The Dirac structure is rewritten as $\{(f, e) \mid W_f f = 0,\; W_e e = 0\}$, which depends only on the row spaces of $W_f$ and $W_e$.
    The rows can be orthonormalized within each block, and then $W = \begin{bsmallmatrix} W_f \\ W_e \end{bsmallmatrix}$ is an orthogonal matrix.

    If $n_f n_e = 0$, take $\Omega = 0$.
    Given a fixed orthogonal matrix $W_0$, applying the CS decomposition to $W W_0^\top$ gives
    \begin{equation}
        W W_0^\top = U \exp\big(S(\Theta)\big) V^\top, \qquad U = \operatorname{diag}(U_f,U_e), \qquad V = \operatorname{diag}(V_f,V_e),
    \end{equation}
    where $U_f,V_f \in \R^{n_f \times n_f}$ and $U_e,V_e \in \R^{n_e \times n_e}$ are orthogonal, and $\Theta \in \R^{n_f \times n_e}$ is zero except for the diagonal entries $\Theta_{ii} = \theta_i \in [0,\pi/2]$ for $i \le \min(n_f,n_e)$.
    Here, $\exp\big(S(\Theta)\big)$ is the cosine--sine factor, which rotates each paired flow and effort coordinate and fixes the unpaired coordinates.
    Set $\Omega := V_f \Theta V_e^\top$.
    Since $S(\Omega) = V S(\Theta) V^\top$,
    \begin{equation}
        W = U V^\top \exp\big(S(\Omega)\big) W_0 = \operatorname{diag}(U_f V_f^\top,U_e V_e^\top)\, W(\Omega).
    \end{equation}
    The block-diagonal coefficient matrix $\operatorname{diag}(U_f V_f^\top,U_e V_e^\top)$ only recombines the flow rows and the effort rows separately and therefore leaves both constraint kernels unchanged.
    Thus, $W(\Omega)$ defines $\gD_\star$, that is, $\gD_\star = \gD(\Omega)$.
\end{proof}

\begin{proof}[Proof of Proposition~\ref{prop:energy}]
    By Theorem~\ref{thm:dirac_parameterize}~(i), $(F, E)$ defines a Dirac structure for every $\Omega$, so Eq.~\eqref{eq:dirac-dae} and condition~(i) of Definition~\ref{def:dirac} give $\langle e_S \mid f_S \rangle + \langle e_R \mid f_R \rangle + \langle e_P \mid f_P \rangle = 0$ along every solution.
    For the $i$-th storage element, the energy-storage relation of either type gives $e_{S,i} f_{S,i} = -\frac{\partial H_{\theta,i}}{\partial z_i} \dot{z}_i = -\frac{\mathrm{d}}{\mathrm{d}t} H_{\theta,i}(z_i)$, and summing over $i$ yields $\langle e_S \mid f_S \rangle = -\frac{\mathrm{d}}{\mathrm{d}t} H_\theta(z)$, which proves Eq.~\eqref{eq:energy-balance}.
    The inequality follows from $\langle e_R \mid f_R \rangle = \sum_j f_{R,j} \hat{\gR}_{\theta,j}(f_{R,j}) \le 0$.
\end{proof}

\section{Implementation of DINEs}
\label{app:implementation}
This appendix describes the implementation of DINEs independently of the tasks; the values of the settings that depend on the task are listed in Appendix~\ref{app:experiments}.
\subsection{Learnable elements and interconnection}
\label{app:dines}
\paragraph{Elements.}
How each element is implemented depends on the prior knowledge given about it, that is, on whether the element is assumed to be linear (see Table~\ref{tab:prior}).
If a storage element is assumed to be linear, its energy is the quadratic function $H_{\theta,i}(z_i) = \frac{1}{2} 10^{\vartheta_i} z_i^2$ with a learned exponent $\vartheta_i$, so that the coefficient is positive.
For nonlinear storage elements, a neural network $\mathrm{NN}_{\theta,i}$ is added to the quadratic term as $H_{\theta,i}(z_i) = \mathrm{NN}_{\theta,i}(z_i) + \frac{1}{2} 10^{\vartheta_i} z_i^2$.

The resistive relation of every resistive element is an explicit map from the flow to the effort, the sum of a neural network and a linear term, $e = \hat{\gR}_{\theta,j}(f) = \mathrm{NN}_{\theta,j}(f) - 10^{\vartheta_j} f$ in our sign convention.
The last affine layer of every network and every exponent $\vartheta$ are initialized to zero, so that at initialization the network outputs vanish and every element is linear with unit coefficient.
Each element network has two hidden layers of 200 units with the hyperbolic tangent, with hidden-layer weights and biases initialized uniformly on $[-1/\sqrt{d_{\mathrm{in}}}, 1/\sqrt{d_{\mathrm{in}}}]$, where $d_{\mathrm{in}}$ is the layer's input dimension.
No sign constraint is imposed on the resistive relation for simplicity.

\paragraph{Interconnection.}
The matrices $(F, E)$ are parameterized as in Theorem~\ref{thm:dirac_parameterize}.
The orthogonal matrix $W_0$ is drawn once at initialization as a random orthogonal matrix, namely the orthogonal factor of the QR decomposition of a matrix with independent standard normal entries~\citep{mezzadri2007how}, and $\Omega$ is initialized to zero. The true $(F, E)$ are never used.

\subsection{Learning and prediction}
\label{app:learning}
\paragraph{Trajectory network.}
The trajectory network $\chi_\eta$ takes the time $t$, normalized by the duration $T$ of the training trajectories, and the index of the trajectory.
The index is embedded into a 256-dimensional space.
The time is encoded by random Fourier features~\citep{tancik2020fourier} at three scales, $1$, $10$, and $100$, following the multi-scale construction of \citet{wang2021eigenvector}.
At each scale $s \in \{1, 10, 100\}$, the time is mapped to 128 sine and 128 cosine features, $\sin(\omega_\ell t / T)$ and $\cos(\omega_\ell t / T)$, with frequencies drawn independently from $\mathcal{N}(0, s^2)$ and held fixed during training.
These features are concatenated with the index embedding and passed through a subnetwork shared by the three scales, with four hidden layers of 256 units and the hyperbolic tangent.
The outputs of the three scales are concatenated and mapped by a final affine layer to the outputs of $\chi_\eta$.
The outputs are the state of every storage element, the flow of every resistive element, and the variable conjugate to the prescribed variable of every source port, whether observed or not.
The affine layers are initialized by the Glorot uniform scheme with zero biases.
The trainable trajectory embeddings are initialized with independent standard normal entries.

\paragraph{Loss.}
In Eq.~\eqref{eq:loss}, each squared norm is implemented as a mean of squares: over the observed channels for the observation term, and over the $n$ rows of the residual $F f + E e$ for the residual term. The observation error is not normalized by the standard deviations of the channels. The weight is $\lambda = 1$ unless stated otherwise.
At every update, the observation term is evaluated on 2{,}000 pairs of a trajectory and one of its sample times, drawn uniformly with replacement from all such pairs, and the residual term on 2{,}000 pairs of a trajectory and a collocation time drawn uniformly over the training duration.

\paragraph{Regularization of unobserved flows.}
When the flow of a resistive element is not observed, its scale is a gauge freedom that the observations do not determine.
To avoid vanishing, a regularization term that pushes the variance of each such flow over the minibatch toward one is added to the loss (see Appendix~\ref{app:task2} for the elements to which it is applied),
\begin{equation}\textstyle
    \mathcal{L}_{\mathrm{var}} = \frac{0.1}{|\mathcal{J}|} \sum_{j \in \mathcal{J}} \big( \operatorname{Var}_{\mathcal{B}}[f_{R,j}] - 1 \big)^2,\label{eq:regularizer}
\end{equation}
where $\mathcal{J}$ is the set of these elements, $\mathcal{B}$ is the minibatch, and $\operatorname{Var}_{\mathcal{B}}$ is the population variance over the minibatch.

\paragraph{Initial state.}
Before prediction, the initial state $z_0$, and the initial rate $\dot{z}_0$, and the algebraic variables $a_0$ (the flows of the resistive elements and the conjugate variables of the source ports) are estimated from the observations $y_0$ and the input $u_0$ at the first time step by minimizing the observation error $\frac{1}{2} \| \hat{y}(z_0, \dot{z}_0, a_0; u_0) - y_0 \|^2$ subject to the learned DAE at $t_0$.
When needed, the derivative of the learned DAE is also used as a constraint.
When a history of $n_h$ samples of inputs and observations is available, the unknowns at all time steps of the history are estimated by minimizing the sum of the observation errors over the history, subject to the learned DAE and the trapezoidal rule between consecutive samples.

\paragraph{Integration.}
At prediction, the learned DAE is integrated from the initial state by the three-stage Radau IIA method of order five~\citep{hairer1996solving} with the fixed step equal to the sampling interval $\Delta t$.
At each step, the stage equations are solved simultaneously by Newton's method.
The prescribed variables of the source ports, which are given as functions of time, are evaluated analytically at every stage.

\subsection{Isolation and composition}
\label{app:coupling}
\paragraph{Gauge of the boundary variable.}
When the boundary port of a subsystem is connected to a source during training, the conjugate boundary variable is an output of $\chi_\eta$ during training and, for prediction, is computed from the learned relations.
Its scale is not determined by the observations of the subsystem (see Section~\ref{sec:discussion}), whereas the composition equates the boundary variables of two subsystems.
Before composition, this scale is therefore calibrated, without retraining, by the following rules.
The calibration is applied consistently to $(F, E)$, the constitutive relations, and the momenta, preserving the observed trajectories of each subsystem.

\paragraph{Fixing the scale from the interconnection.}
For junction interconnections, the nonzero connection coefficients are known a priori to have unit magnitude. This prior is used to fix scales and is not imposed on coefficients representing unknown transformer ratios.
When the learned interconnection, read in reduced row echelon form, has a row that relates the boundary variable, whose scale is undetermined, to an observed variable, whose physical scale is known, through a junction, that is, with a coefficient of unit magnitude, the scale is fixed by normalizing the absolute value of the learned coefficient to $1$.
In this case, no additional observation of the boundary variable is needed for the calibration.
When a single element is isolated for reuse, its scale is fixed in the same way from the learned interconnection of the subsystem in which it was learned, before the element is taken out.
When the row involves several observed variables, the coefficient of the largest magnitude is normalized.

\paragraph{Fixing the scale from measurements.}
Otherwise, when the scale cannot be fixed from the interconnection and the existing observations, it is calibrated from additional measurements of the boundary variable.
With the boundary effort $\hat{e}_\mathrm{b}$ predicted by the model and the measured boundary effort $e_\mathrm{b}^{\mathrm{obs}}$, the scale $\alpha$ and the intercept $\beta$ are estimated by
    \begin{equation*}\textstyle
        (\alpha, \beta) = \operatorname*{arg\,min}_{a, b} \sum_{j=1}^{N_{\mathrm{cal}}} \big( a \hat{e}_{\mathrm{b}, j} + b - e_{\mathrm{b}, j}^{\mathrm{obs}} \big)^2,
    \end{equation*}
    where $N_{\mathrm{cal}}$ is the number of measurements, and the calibrated boundary effort is $\tilde{e}_\mathrm{b} = \alpha \hat{e}_\mathrm{b} + \beta$.
The intercept is a constant shift of the boundary effort. Since the effort rows $E e = 0$ of the DAE are preserved by any shift $\delta e$ of the efforts with $E\, \delta e = 0$, the intercept is absorbed by constant shifts $c_i$ of the efforts of the C-type storage elements, $\tilde{e}_{C,i} = \alpha e_{C,i} + c_i$, which is equivalent to replacing the energy of each C-type storage element by $\tilde H_i(q_i) = \alpha H_i(q_i) + c_i q_i$; the shifts are chosen so that $E\, \delta e = 0$ holds for the rescaled matrix $E$ and $\delta e$ with $c_i$ at the ports of the C-type elements, $\beta$ at the boundary port, and $0$ at the other ports.

\paragraph{Fixing the Transformer ratio from measurements.}
Similarly, if two subsystems are connected through a transformer with an unknown ratio, the ratio is identified after the scales of both subsystems are fixed.
The boundary outputs of the two subsystems are evaluated at a few states of the composed system, and the ratio is obtained by closed-form least squares from the relation of the transformer.
The ratio is then frozen.
This procedure is required for the piston area of Task~4, mentioned in Appendix~\ref{app:task4}.

\paragraph{Composed DAE.}
The composed system is the DAE formed by the frozen DAEs of the subsystems and the connection; the boundary variables are unknowns of this DAE and are solved at every stage of the Radau IIA method together with the states.
No data of the composed system are used to retrain the subsystems.

\subsection{Ablations}
\label{app:ablation}
\paragraph{DINEs (soft).}
DINEs (soft) replace the exact parameterization of Theorem~\ref{thm:dirac_parameterize} by independently learned matrices $F, E \in \R^{n \times n}$ and soft penalties.
The loss of Eq.~\eqref{eq:loss} is augmented by
\begin{equation*}\textstyle
    \mathcal{L}_1 = \big\| E F^\top + F E^\top \big\|_{\mathrm{F}}^2, \qquad \mathcal{L}_2 = \big\| F F^\top + E E^\top - I \big\|_{\mathrm{F}}^2 ,
\end{equation*}
where $\| \cdot \|_{\mathrm{F}}$ denotes the Frobenius norm, and both penalties have unit weight.
The second term penalizes deviations from row orthonormality of $[\, F \mid E \,]$; exact row orthonormality implies the rank condition of Eq.~\eqref{eq:conditions}.
Both penalties vanish identically under the parameterization of Theorem~\ref{thm:dirac_parameterize} (see Appendix~\ref{app:proofs} for proof).
The entries of $F$ and $E$ are initialized independently from a Gaussian distribution with mean zero and variance $1 / (2n)$.
No block structure or value of $n_f$ is enforced.
The remaining training settings and the initial-state estimation and time-integration procedures are unchanged.
The learned matrices are used directly for prediction, without projection onto a Dirac structure.
For this ablation, ``not converged'' in Tables~\ref{tab:results} and~\ref{tab:thresholds} denotes a numerical failure during initial-state estimation or time integration that prevents completion of all ten test trajectories for a seed.
This designation is not based on a threshold on the training loss or prediction error.
Means and standard deviations are reported only over seeds that complete all ten trajectories, with the number of excluded seeds stated in the tables.

\paragraph{Misspecified $n_f$.}
DINEs with a misspecified $n_f$ are the model with $n_f$ set to a value different from that of the true interconnection and $n_e = n - n_f$; all other settings are unchanged.
By Theorem~\ref{thm:dirac_parameterize}, the true Dirac structure is then outside the model family, so that the misspecification leads to a catastrophic failure rather than to a silently wrong model (see Section~\ref{sec:discussion}).
The values of $n_f$ examined are listed in Appendix~\ref{app:task1}.

\section{Experiments}
\label{app:experiments}
\subsection{Common settings}
\label{app:common}
\paragraph{Systems.}
Table~\ref{tab:tasks} lists the elements, observations, and inputs of the learned systems of Tasks~1--4.
\begin{table}[t]
    \centering
    \caption{Systems of Tasks 1--4.}
    \label{tab:tasks}
    \small
    \setlength{\tabcolsep}{4pt}
    \begin{tabularx}{\linewidth}{lcccccccYY}
        \toprule
        Task: system             & C & I  & R  & Se & Sf & $n$ & $(n_f, n_e)$ & Observations                & Inputs                         \\
        \midrule
        1: Three-way junction    & 3 & 3  & -- & -- & -- & 6   & (2, 4)       & 6 states                    & --                             \\
        2: RC ladder             & 4 & -- & 7  & 2  & -- & 13  & (8, 5)       & 10 of 13 currents           & 2 source voltages              \\
        3: Subsystem A           & 3 & 3  & 2  & 1  & -- & 9   & (6, 3)       & 6 states                    & boundary force                 \\
        \phantom{3:} Subsystem B & 1 & 2  & 1  & 2  & -- & 6   & (4, 2)       & 3 states                    & boundary and external forces   \\
        4: Hydraulic             & 1 & 1  & 1  & 1  & 1  & 5   & (3, 2)       & 2 states                    & source pressure, boundary flow \\
        \phantom{4:} Mechanical  & 3 & 2  & 2  & -- & 1  & 8   & (5, 3)       & 5 states                    & boundary velocity              \\
        \bottomrule
    \end{tabularx}
    \par\smallskip
    \begin{minipage}{\linewidth}
        \footnotesize
        Boundary variables are recorded separately for training the boundary networks of Coupled Neural ODEs and, for the mechanical subsystem of Task~4, calibrating DINEs (see Table~\ref{tab:prior}).
    \end{minipage}
\end{table}
The coefficients of the true systems below are disclosed for reproducibility and are not given to any model.

\paragraph{Sampling and data.}
All trajectories are sampled with the time step $\Delta t = 0.1$ in the units of the simulation, without noise. Durations are counted in steps, so that a trajectory of $K$ steps has $K + 1$ samples, and the frequencies of the inputs are given in cycles per step.
Every training set, of Tasks~1 and~2 and of each subsystem of Tasks~3 and~4, consists of 100 trajectories of 500 steps, and every test set of 10 trajectories of 10{,}000 steps.

\paragraph{Reference trajectories.}
The reference trajectories are generated by integrating the true equations with the eighth-order Runge--Kutta method DOP853~\citep{hairer1993solving} with the relative tolerance $10^{-11}$, the absolute tolerance $10^{-13}$, and the maximum step $\Delta t$.
The generated data are checked against a second integration with the tolerances $10^{-12}$ and $10^{-14}$ and the maximum step $\Delta t / 2$.

\paragraph{Training.}
All models are trained with five seeds, which change the initialization and the sampling of minibatches, with minibatches of 2{,}000.
The optimizer is Adam~\citep{kingma2015adam} with the learning rate $10^{-3}$, $(\beta_1, \beta_2) = (0.9, 0.999)$, and no weight decay, with cosine annealing~\citep{loshchilov2017sgdr} of the learning rate to zero over all updates.
The number of updates is 100{,}000 unless stated otherwise.
Training is in single precision. Whenever a loss is normalized by standard deviations, these are the sample standard deviations over the training set. The parameters after the last update are used, without early stopping and without model selection on the test set.

\paragraph{Prediction.}
Predictions are computed in double precision.
The learned DAEs of DINEs, the composed systems, and the receiving system of Task~5 are integrated as described in Appendices~\ref{app:learning} and~\ref{app:coupling}.
The ODE models are integrated by the adaptive Dormand--Prince method with the relative tolerance $10^{-9}$ and the absolute tolerance $10^{-11}$, the maximum step $\Delta t$, and outputs at every step.

\paragraph{Evaluation metrics.}
The metrics of Section~\ref{sec:settings} are defined as follows.
Let $\hat{y}(t_k)$ and $y(t_k)$ be the predicted and true observations at the $k$-th sample and $d$ the number of observed channels. The squared error at sample $k$ and the MSE over an evaluation horizon of $K$ steps ($K = 10{,}000$ unless stated otherwise), including the initial sample, are
\begin{equation}
    e_k = \frac{1}{d} \big\| \hat{y}(t_k) - y(t_k) \big\|^2,
    \qquad
    \mathrm{MSE} = \frac{1}{K + 1} \sum_{k=0}^{K} e_k .
    \label{eq:mse}
\end{equation}
The valid prediction time (VPT)~\citep{vlachas2020backpropagation} for a threshold $\varepsilon$ is the fraction of the horizon before $e_k$ first exceeds $\varepsilon$: with $k_*$ the first index such that $e_{k_*} > \varepsilon$,
\begin{equation}
    \mathrm{VPT}(\varepsilon) = \frac{\max(k_* - 1, 0)}{K},
    \label{eq:vpt}
\end{equation}
and $\mathrm{VPT}(\varepsilon) = 1$ if the threshold is never exceeded.
Both metrics are averaged over the ten test trajectories of each seed.

\subsection{Task 1: three-way junction}
\label{app:task1}
\paragraph{System.}
Three masses $m_i$ are attached through springs to one massless junction.
With the elongation $q_i$ of spring $i$, the velocity $v_i$ of mass $i$, and the velocity $v_J$ of the junction,
\begin{equation*}
    \dot{q}_i = v_i - v_J, \qquad
    m_i \dot{v}_i = -e_i(q_i), \qquad
    0 = \sum_{i=1}^{3} e_i(q_i), \qquad i = 1, 2, 3,
\end{equation*}
where the last equation is the force balance at the junction.
The spring forces are $e_i(q_i) = a_i q_i + b_i q_i^3$ with $(a_i, b_i) = (0.4, 0.05), (0.3, 0.4), (0.25, 0.15)$, and the masses are $(m_1, m_2, m_3) = (1.0, 1.2, 0.8)$; there are no dampers and no inputs.
For the reference trajectories, the force balance is differentiated in time, which gives $v_J = \sum_i k_i(q_i) v_i / \sum_i k_i(q_i)$ with the tangent stiffness $k_i(q_i) = a_i + 3 b_i q_i^2$; this $v_J$ is substituted, and the resulting ODE is integrated.
In port-based terms, the system has three C-type and three I-type elements, and the junction imposes two constraints among the flows and four among the efforts.

\paragraph{Data.}
The initial elongations $q_1$ and $q_2$ are uniform on $[-1.2, 1.2]$ and the initial velocities uniform on $[-0.6, 0.6]$; $q_3$ is solved from the force balance, whose solution is unique since $e_3$ is monotonic.
All six states $(q_1, q_2, q_3, v_1, v_2, v_3)$ are observed.
The initial states of the test trajectories are drawn from the same distributions as those of the training trajectories.

\paragraph{DINEs.}
The I-type elements are assumed to be linear, with their masses learned, and the springs are nonlinear.
The main setting is $(n_f, n_e) = (2, 4)$; in the sweep, $n_f \in \{1, 2, 3, 4, 5\}$ with $n_e = 6 - n_f$, and all other settings, including the initialization rule and the observations, are unchanged.
DINEs estimate the initial state from the six observations at the first sample.

\paragraph{Neural ODEs and PoDiNNs.}
\label{app:baselines}
Neural ODEs~\citep{chen2018neural} take the observed states as the state $x$ and learn the vector field by a network with two hidden layers of 200 units and the hyperbolic tangent.
PoDiNNs~\citep{podinn} consist of one network per C-type element for its energy, of the same multilayer-perceptron structure as the element networks of DINEs but without an explicit quadratic term, following the original implementation, one learned mass $10^{\vartheta_i}$ per I-type element, and a constant matrix $B_\theta \in \R^{n_I \times n_C}$, initialized to zero, that couples the two:
\begin{equation*}
    \dot{q} = B_\theta^\top v,
    \qquad
    \dot{v} = -M^{-1} B_\theta\, \nabla_q H_\theta(q),
\end{equation*}
where $M$ is the diagonal matrix of the masses, and $n_C$ and $n_I$ are the numbers of C-type and I-type elements.
Here, $B_\theta$ is the I--C coupling block of the matrix of the input-output representation in Section~\ref{sec:related}; on Task~1, which has no resistive elements and no sources, the coupling parameters of the original implementation reduce to this zero-initialized constant matrix.
PoDiNNs~\citep{podinn} are given the list of elements, the correspondence between the observed channels and the states of the elements, and the linearity of the I-type elements.
Neither the constraint nor the state-dependent matrix $B(q)$ of the reduced ODE is supplied, and the states are not projected onto the true constraints.
Both models are trained by one-step prediction: from a training sample $x_k$, the state after $\Delta t$ is predicted by one step of the classical fourth-order Runge--Kutta method.
The loss is the squared error to $x_{k+1}$, normalized by the sample standard deviation of each channel over the training set and averaged over the channels and the minibatch.
Predictions start from the observed initial state, $x_0 = y_0$, and integrate the learned ODE by the adaptive Dormand--Prince method~\citep{dormand1980family}.

\subsection{Task 2: RC ladder}
\label{app:task2}
\paragraph{System.}
Two voltage sources $u_\mathrm{A}$ and $u_\mathrm{B}$, two nodes $a$ and $b$, four grounded linear capacitors $\mathrm{C}_1, \dots, \mathrm{C}_4$, and seven nonlinear resistors $\mathrm{R}_0, \dots, \mathrm{R}_6$ are connected as follows: $\mathrm{R}_0$ connects source A to node $a$, $\mathrm{R}_1$ and $\mathrm{R}_2$ connect node $a$ to $\mathrm{C}_1$ and $\mathrm{C}_2$, $\mathrm{R}_5$ connects node $a$ to node $b$, $\mathrm{R}_3$ and $\mathrm{R}_4$ connect node $b$ to $\mathrm{C}_3$ and $\mathrm{C}_4$, and $\mathrm{R}_6$ connects node $b$ to source B.
In Figure~\ref{fig:systems}, the source voltages are denoted by $V_{E_1} = u_\mathrm{A}$ and $V_{E_2} = u_\mathrm{B}$.
Capacitor $i$ has the charge $q_i$ and the voltage $v_i = \gamma_i q_i$ with $(\gamma_1, \dots, \gamma_4) = (0.375, 2.1, 0.45, 2.4)$, that is, the capacitances are $1 / \gamma_i$. Resistor $j$ has the voltage drop $\Delta V_j = r_j I_j + k_j I_j^3$ in the direction of Table~\ref{tab:rc}. The currents satisfy $I_0 = I_1 + I_2 + I_5$ and $I_5 = I_3 + I_4 + I_6$ at the two nodes, and the charges satisfy $\dot{q} = (I_1, I_2, I_3, I_4)$.
The node voltages are algebraic variables determined by Kirchhoff's current laws together with the nonlinear resistor relations.
For the reference trajectories, the resulting nonlinear equations are solved by Newton's method at every evaluation.
In port-based terms, the system has 13 ports, four C-type elements, seven resistive elements, and two effort sources; its graph has 13 branches and nine nodes, hence eight current-law constraints and five voltage-law constraints.
\begin{table}[t]
    \centering
    \caption{Resistors of Task 2.}
    \label{tab:rc}
    \small
    \begin{tabular}{lccc}
        \toprule
        Resistor       & $\Delta V_j$       & $r_j$  & $k_j$  \\
        \midrule
        $\mathrm{R}_0$ & $u_\mathrm{A} - a$ & 0.8299 & 0.7991 \\
        $\mathrm{R}_1$ & $a - v_1$          & 0.6711 & 4.6312 \\
        $\mathrm{R}_2$ & $a - v_2$          & 0.7407 & 6.3965 \\
        $\mathrm{R}_3$ & $b - v_3$          & 0.6981 & 6.5022 \\
        $\mathrm{R}_4$ & $b - v_4$          & 0.5908 & 5.1852 \\
        $\mathrm{R}_5$ & $a - b$            & 0.7813 & 2.3884 \\
        $\mathrm{R}_6$ & $b - u_\mathrm{B}$ & 0.7018 & 0.7792 \\
        \bottomrule
    \end{tabular}
\end{table}

\paragraph{Inputs and data.}
Each source voltage is a sum of six cosines and an offset, $u(t) = c + \sum_{\ell=1}^{6} A_\ell \cos(2\pi \nu_\ell\, t / \Delta t + \phi_\ell)$, where the frequencies $\nu_\ell$ are in cycles per step, so that $\nu_\ell / \Delta t$ is the frequency per unit time; the inputs are evaluated analytically at all sample, collocation, and stage times.
The two sources are generated as a common mode and a differential mode, $u_\mathrm{A} = u_\mathrm{c} + u_\mathrm{d}$ and $u_\mathrm{B} = u_\mathrm{c} - u_\mathrm{d}$, each of which is a sum of three cosines and an offset.
The trajectories are divided into six groups, which combine the upper frequency $\nu_{\max} \in \{0.03, 0.08\}$ with the level $\kappa \in \{4, 12, 20\}$, with the group labels randomly permuted over the trajectories; for each $\nu_{\max}$, the training groups contain 17, 17, and 16 trajectories for the three increasing levels and the test groups 2, 2, and 1 trajectories.
In each group, the frequencies of the common mode are uniform on the lower half $[0.0025, (0.0025 + \nu_{\max}) / 2]$ of the band and those of the differential mode on the upper half $[(0.0025 + \nu_{\max}) / 2, \nu_{\max}]$, the phases are uniform on $[0, 2\pi]$, and, before the rescaling below, the offset of each mode is uniform on $[-0.5 \kappa, 0.5 \kappa]$.
The rescaling gives the two modes the same mean square on a fixed grid.
Before the rescaling, the amplitudes are uniform on $[0.25 \kappa, 0.65 \kappa]$ for the common mode and on $[0.2 \kappa, 0.52 \kappa]$ for the differential mode.
For the amplitude normalization, provisional frequencies $\nu^{(0)}_{m,\ell}$ are first drawn independently and uniformly from the full band $[\nu_{\min},\nu_{\max}]$ for both modes, where $\nu_{\min}=0.0025$ and $m\in\{c,d\}$.
Using these provisional frequencies and the original phases, the sum $S$ of the mean squares of the two modes, including their offsets, is computed on the grid $t_j=j\Delta t/2$, $j=0,\ldots,500$.
The frequencies are then mapped to the respective half-bands by
\begin{equation*}
    \nu_{c,\ell}
    =\frac{\nu_{\min}+\nu^{(0)}_{c,\ell}}{2},
    \qquad
    \nu_{d,\ell}
    =\frac{\nu_{\max}+\nu^{(0)}_{d,\ell}}{2},
\end{equation*}
without drawing new frequencies, and the differential-mode phases are shifted by $\pi/2$.
Finally, the amplitudes and offset of each mode are multiplied by a single factor specific to that mode, so that its mean square, evaluated with the remapped frequencies and updated phases on the same grid, equals $S/4$.
The initial capacitor voltages are obtained from $\tilde{v}_i = a^0_i \xi_i$ with $\xi_i$ uniform on $[-1.1, 1.1]$ and $a^0 = (0.75, 1.05, 0.9, 1.2)$: within the pairs $(1, 2)$ and $(3, 4)$, the mean of the two voltages is kept and their difference is tripled, which gives $\hat{v}_i$, and then $v_i(0) = s \hat{v}_i$ with $s = \big( \sum_i \tilde{v}_i^2 / a^0_i \big)^{1/2} \big( \sum_i \hat{v}_i^2 / a^0_i \big)^{-1/2}$, so that the weighted sum of squares $\sum_i v_i(0)^2 / a^0_i$ equals that of $\tilde{v}$. This transformation increases the within-pair voltage differences relative to their common components while preserving the specified weighted sum of squares.
The initial charges are $q_i(0) = v_i(0) / \gamma_i$.
The initial states of the test trajectories are drawn from the same distributions as those of the training trajectories.
The observed channels are the currents at ten of the 13 ports: the source currents $I_\mathrm{A} = I_0$ and $I_\mathrm{B} = -I_6$, the currents of the four capacitors, and the currents of $\mathrm{R}_1, \dots, \mathrm{R}_4$; all ten are given to all models alike.
The current labels in Figure~\ref{fig:systems} satisfy $I_{E_1} = I_\mathrm{A} = I_0$, $I_{E_2} = I_\mathrm{B} = -I_6$, and $I_{C_i} = I_{R_i} = I_i = \dot{q}_i$ for $i = 1, \dots, 4$.
The charges, the node voltages, and the currents of $\mathrm{R}_0$, $\mathrm{R}_5$, and $\mathrm{R}_6$ are not observed.
That the capacitor currents coincide with the currents of $\mathrm{R}_1, \dots, \mathrm{R}_4$, and $I_\mathrm{A}$ and $I_\mathrm{B}$ with the currents of $\mathrm{R}_0$ and $\mathrm{R}_6$ up to sign, is not given to any model.
Each test trajectory has 20 samples of inputs and observations before the prediction as the history.

\paragraph{DINEs.}
The capacitors are assumed to be linear with unknown capacitances, and the resistors are nonlinear.
DINEs use $\lambda = 10$, the regularization term $\mathcal{L}_{\mathrm{var}}$ (as defined in Appendix~\ref{app:learning}) for the flows of the three unobserved resistors $\mathrm{R}_0$, $\mathrm{R}_5$, and $\mathrm{R}_6$, and 500{,}000 updates.
The initial state is estimated from the history of 20 samples, starting from zero charges.
The model is then advanced by one step from the last sample of the history, and the following 10{,}000 steps are evaluated, including the first predicted sample.

\paragraph{CT-SUBNET.}
\label{app:ctsubnet}
CT-SUBNET~\citep{beintema2023continuous} consists of an encoder $\psi_\theta$ that maps the inputs and outputs of the $n_h = 20$ past samples to the state $z \in \R^{n_z}$, a vector field $\dot{z} = \varphi_\theta(z, u)$, and an observation $\hat{y} = h_\theta(z, u)$.
The observation map takes the input as well, unlike the original form $h_\theta(z)$, because the observed currents depend directly on the source voltages; otherwise the currents would have to be reproduced from the latent state alone.
Each of the three maps is the sum of a network with two hidden layers of 64 units and the hyperbolic tangent and an affine map from its input to its output, with zero initial biases.
The inputs and outputs are standardized by their means and sample standard deviations over the training set, and the output of $h_\theta$ is transformed back to physical units.
The output of the network of the vector field is divided by a time scale of $50 \Delta t$.
At every update, 2{,}000 windows are drawn from the training set, each consisting of the $n_h$ past samples for the encoder and the following 50 samples.
Over the 50 samples, the model is integrated by the classical fourth-order Runge--Kutta method with one step per $\Delta t$ and the analytic inputs at the stage times.
The mean squared error of the 50 predicted outputs, normalized by the sample standard deviation of each output channel, is minimized with respect to $\psi_\theta$, $\varphi_\theta$, and $h_\theta$ jointly.
For prediction, the encoder maps the history to the state at the first predicted sample once, and the model is integrated with the known inputs by the adaptive Dormand--Prince method; neither the encoder nor the state is updated by observations afterwards.
The state dimension is $n_z \in \{4, 8, 12, 16, 20\}$, and the model is trained for 500{,}000 updates.

\subsection{Task 3: rigid coupling}
\label{app:task3}
\paragraph{System.}
Each subsystem is a chain of masses, springs, and dampers with one boundary port.
Springs have $e = a q + b q^3$, and dampers have $e = -R(f)$ with $R(f) = a f + b f^3$, in our sign convention; the equations of motion below are written with the resistive forces $R$, which are the negatives of the efforts.
Subsystem A is a chain from a wall, a spring with a damper, mass $m_1$, a spring, mass $m_2$, a spring with a damper, and mass $m_3$, to which the boundary force $u_\mathrm{A}$ is applied.
With the absolute positions $x_1, x_2, x_3$ of the masses, the elongations of the three branches are $(q_1, q_2, q_3) = (x_1, x_2 - x_1, x_3 - x_2)$, which are observed together with the velocities $v_1, v_2, v_3$, and the equations of motion are
\begin{equation*}
    \begin{aligned}
         & \dot{q}_1 = v_1, \qquad \dot{q}_2 = v_2 - v_1, \qquad \dot{q}_3 = v_3 - v_2,                               \\
         & m_1 \dot{v}_1 = -T_1 + T_2, \qquad m_2 \dot{v}_2 = -T_2 + T_3, \qquad m_3 \dot{v}_3 = -T_3 + u_\mathrm{A},
    \end{aligned}
\end{equation*}
where $T_i$ is the total force of the spring and the damper of branch $i$, $T_i = e_i(q_i) + R_i(\dot{q}_i)$ for $i = 1, 3$ and $T_2 = e_2(q_2)$.
Subsystem B consists of mass $m_{\mathrm{B}1}$, to which the boundary force $u_{\mathrm{B}0}$ is applied, a spring with a damper, and mass $m_{\mathrm{B}2}$, to which the external force $u_{\mathrm{B}1}$ is applied.
With $q_\mathrm{B} = x_{\mathrm{B}2} - x_{\mathrm{B}1}$ and $T_\mathrm{B} = e_\mathrm{B}(q_\mathrm{B}) + R_\mathrm{B}(v_{\mathrm{B}2} - v_{\mathrm{B}1})$,
\begin{equation*}
    \dot{q}_\mathrm{B} = v_{\mathrm{B}2} - v_{\mathrm{B}1}, \qquad
    m_{\mathrm{B}1} \dot{v}_{\mathrm{B}1} = T_\mathrm{B} - u_{\mathrm{B}0}, \qquad
    m_{\mathrm{B}2} \dot{v}_{\mathrm{B}2} = -T_\mathrm{B} - u_{\mathrm{B}1}.
\end{equation*}
The recorded boundary velocities are $v_3$ of A and $v_{\mathrm{B}1}$ of B; that they coincide with observed mass velocities is not given to any model.
Table~\ref{tab:subsystems} lists the elements of the subsystems.
\begin{table}[t]
    \centering
    \caption{Subsystems of Tasks 3 and 4.}
    \label{tab:subsystems}
    \small
    \setlength{\tabcolsep}{3pt}
    \begin{tabular}{llllll}
        \toprule
        Subsystem          & Inertias      & Springs $(a, b)$                                                                 & Dampers $(a, b)$                                                       & \begin{tabular}[c]{@{}l@{}}Boundary\\(prescribed $\to$ recorded)\end{tabular} & Observed       \\
        \midrule
        Task 3, A          & $1, 1.2, 1.4$ & \begin{tabular}[t]{@{}l@{}}$(0.4, 0.1)$\\$(0.3, 0.1)$\\$(0.2, 0.1)$\end{tabular} & \begin{tabular}[t]{@{}l@{}}$(0.07, 0.45)$\\$(0.09, 0.55)$\end{tabular} & $F_\mathrm{b} \to v$                                                          & $3\,q$, $3\,v$ \\
        \addlinespace
        Task 3, B          & $0.9, 1.1$    & $(0.35, 0.15)$                                                                   & $(0.08, 0.5)$                                                          & $F_\mathrm{b} \to v$                                                          & $1\,q$, $2\,v$ \\
        \addlinespace
        Task 4, hydraulic  & $0.9$         & $(0.8, 0)$ (reservoir)                                                           & fluid resistor                                                         & $Q_\mathrm{b} \to P$                                                          & $1\,q$, $1\,Q$ \\
        \addlinespace
        Task 4, mechanical & $1, 1.4$      & \begin{tabular}[t]{@{}l@{}}$(0.4, 0.1)$\\$(0.3, 0.1)$\\$(0.2, 0.1)$\end{tabular} & \begin{tabular}[t]{@{}l@{}}$(0.07, 0.45)$\\$(0.09, 0.55)$\end{tabular} & $v \to F_\mathrm{b}$                                                          & $3\,q$, $2\,v$ \\
        \bottomrule
    \end{tabular}
\end{table}

\paragraph{Boundary port.}
During the training of a subsystem, its boundary port is a source: the prescribed boundary variable, a force, is a known input, and the conjugate variable, a velocity, is recorded but excluded from the losses of the dynamics training of every model.
The recorded boundary variables of the training data are used for the training of the boundary network of Coupled Neural ODEs (and, in Task~4, for the calibration of DINEs).
No observation of the boundary variables is given during the prediction of the composed systems.
In the composed system, the boundary velocities are equated, the reaction force $\mu$ acts with $+\mu$ on A and $-\mu$ on B and is an unknown, and $u_{\mathrm{B}1}$ is the only known input; no data of the composed system are used for training.
The composed systems are integrated as DAEs with the boundary variables as unknowns, and the initial state is estimated from the nine observations at the first sample.

\paragraph{Data.}
The initial positions of A are uniform on $[-1, 1]$ and its velocities on $[-0.6, 0.6]$; the initial elongation of B is uniform on $[-0.8, 0.8]$ and its velocities on $[-0.5, 0.5]$.
The boundary force of A is a sum of three cosines with frequencies uniform on $[0.005, 0.02]$, amplitudes uniform on $[0.4, 0.8]$, and phases uniform on $[0, 2\pi]$. Each force of B is a sum of four cosines with frequencies drawn from the multiples of the reciprocal of the training duration within $[0.006, 0.018]$, amplitudes uniform on $[0.08, 0.24]$, and phases uniform on $[0, 2\pi]$.
The reference trajectories of the composed system are generated by substituting the coupling variables, computed analytically, into the differential equations and integrating the resulting ODE in the original states of the subsystems with the method and tolerances of Appendix~\ref{app:common}.
Here, $m_\mathrm{A}$ and $m_\mathrm{B}$ are the two coupled masses, $F_\mathrm{A}$ and $F_\mathrm{B}$ are the forces acting on them other than the reaction force, and the reaction force $\mu = (m_\mathrm{A} F_\mathrm{B} - m_\mathrm{B} F_\mathrm{A}) / (m_\mathrm{A} + m_\mathrm{B})$ keeps the boundary velocities equal; the initial boundary velocities are set equal.

\paragraph{Test trajectories.}
The test trajectories of the composed systems are generated by the following rule.
Each candidate starts from independently selected subsystem training initial states, adjusted to satisfy the coupling constraints.
The input frequencies and amplitudes are taken from the selected training trajectories, with independently redrawn uniform phases.
A candidate is accepted if, at all samples of the 10{,}000 steps of its reference trajectory, every state, known input, recorded boundary variable, and resistive flow stays within the range of the corresponding channel in the training data of the subsystems, and the residual of the coupling constraints, $\max_k |v_3(t_k) - v_{\mathrm{B}1}(t_k)|$ at the samples, stays below $2 \times 10^{-7}$ (an acceptance criterion after generation, not a tolerance of the integrator).
The first ten accepted candidates, out of at most 500, form the test set, which is common to both methods and all seeds.

\paragraph{DINEs.}
The I-type elements are assumed to be linear with unknown coefficients, and the springs and dampers are nonlinear.
The scale of the recorded boundary velocity of each subsystem is fixed from the learned interconnection as described in Appendix~\ref{app:coupling}, without additional observations of the boundary variable.

\paragraph{Coupled Neural ODEs.}
\label{app:cnode}
\textbf{Subsystem dynamics.}
For each subsystem, a Neural ODE of Appendix~\ref{app:baselines} takes the observed states as its state $x$ and the prescribed boundary variable, together with the other known inputs, as its input $u$, $\dot{x} = f_\theta(x, u)$, and is trained by the same one-step loss, with the known inputs evaluated analytically at the stage times.

\textbf{Boundary network.}
Composition also requires the boundary variable conjugate to the prescribed boundary input, so a separate network $h_\phi(x)$ with two hidden layers of 64 units, the hyperbolic tangent, and one scalar output predicts it from the current state of the subsystem only: the boundary velocity on Task~3.
After the training of $f_\theta$, its parameters are frozen and $h_\phi$ is trained separately on the recorded states $x_j^{\mathrm{obs}}$ and the recorded boundary variables $b_j^{\mathrm{obs}}$ of the training data of the subsystem, not on states generated by the learned vector field.
At every update, a minibatch $\mathcal{B}$ of 2{,}000 pairs is drawn with replacement, and
\begin{equation*}\textstyle
    \mathcal{L}_{\mathrm{boundary}} = \frac{1}{|\mathcal{B}|} \sum_{j \in \mathcal{B}} \Big( \frac{h_\phi(x_j^{\mathrm{obs}}) - b_j^{\mathrm{obs}}}{s_b} \Big)^2
\end{equation*}
is minimized, where $s_b$ is the sample standard deviation of the recorded boundary variable over the training set, with the optimizer, the schedule, and the number of updates of Appendix~\ref{app:common}.
The normalization is applied to the loss only, so the network predicts the boundary variable in its recorded physical units.

\textbf{Composition and prediction.}
Both vector fields and both boundary networks are frozen before composition.
During prediction, each boundary network is evaluated on the current predicted state of its subsystem, and its output enters the same connection as for DINEs: on Task~3, the outputs of the two boundary networks are equated as the boundary velocities, and the reaction force is an unknown input of the vector fields.
The composed system is integrated, and its initial state is estimated, using the procedures of Appendices~\ref{app:learning} and~\ref{app:coupling} and the settings stated above; neither the vector fields nor the boundary networks are retrained on data of the composed system.

\subsection{Task 4: hydraulic--mechanical coupling}
\label{app:task4}
\paragraph{System.}
Both subsystems are chains with one boundary port as in Task~3 (see Appendix~\ref{app:task3}), the mechanical subsystem with masses, springs, and dampers, and the hydraulic subsystem with a fluid inertia, a reservoir, and a fluid resistor; the springs and dampers have the same forms as in Task~3, with the coefficients of Table~\ref{tab:subsystems}.
The hydraulic subsystem consists of a pressure source $u_P$, a fluid resistor, one fluid inertia, and a reservoir, with the states $(q_h, Q_1)$ and the known inputs $u_P$ and the boundary flow $Q_\mathrm{b}$:
\begin{equation*}
    P = 0.8\, q_h, \qquad
    \dot{q}_h = Q_1 - Q_\mathrm{b}, \qquad
    0.9\, \dot{Q}_1 = u_P - P - R(Q_1),
\end{equation*}
where $R(Q) = 0.16\, Q \sqrt{Q^2 + (0.18 / 0.16)^2}$; the boundary pressure $P$ is recorded.
The mechanical subsystem is a chain from the boundary with the prescribed velocity $v_\mathrm{b}$, a spring, mass $m_1$, a spring with a damper, mass $m_2$, and a spring with a damper, to a wall, with the states $(q_1, q_2, q_3, v_1, v_2)$ and the input $v_\mathrm{b}$ only:
\begin{equation*}
    \begin{aligned}
         & \dot{q}_1 = v_\mathrm{b} - v_1, \qquad
        \dot{q}_2 = v_1 - v_2, \qquad
        \dot{q}_3 = v_2,                                      \\
         & m_1 \dot{v}_1 = e_1 - e_2 - R_1(v_1 - v_2), \qquad
        m_2 \dot{v}_2 = e_2 - e_3 + R_1(v_1 - v_2) - R_2(v_2),
    \end{aligned}
\end{equation*}
where $e_i$ and $R_i$ are the spring forces and the damper forces of Table~\ref{tab:subsystems}; the boundary force $F_\mathrm{b} = -e_1(q_1)$ is recorded.
The coupling is $Q_\mathrm{b} = A_\mathrm{p} v_\mathrm{b}$ and $A_\mathrm{p} P + F_\mathrm{b} = 0$ with the unknown piston area $A_\mathrm{p}$, whose true value $1.3$ is not given.

\paragraph{Boundary port.}
The protocol of Task~3 applies: the boundary flow $Q_\mathrm{b}$ is prescribed and the pressure $P$ is recorded for the hydraulic subsystem, and the velocity $v_\mathrm{b}$ is prescribed and the force $F_\mathrm{b}$ is recorded for the mechanical subsystem (see Table~\ref{tab:subsystems}).
In the composed system, only the pressure source is a known input, and $Q_\mathrm{b}$ and $v_\mathrm{b}$ are unknowns.
The composed system is integrated as a DAE with the boundary variables as unknowns, and the initial state is estimated from the seven observations at the first sample.

\paragraph{Data.}
The initial states of the hydraulic subsystem are uniform on $[-0.7, 0.7]$ for $q_h$ and on $[-0.45, 0.45]$ for $Q_1$; those of the mechanical subsystem are uniform on $[-0.7, 0.7]$ for $q_1$ and $q_2$, with $q_3 = -q_1 - q_2$, and on $[-0.45, 0.45]$ for $v_1$ and $v_2$.
Every input is a sum of five cosines: the frequency band is divided into five equal sub-bands and one frequency is drawn uniformly from each, the phases are uniform on $[-\pi, \pi]$, and the amplitudes are $A_\ell = s\, b_\ell / \sqrt{\sum_\ell b_\ell^2 / 2}$ with relative amplitudes $b_\ell$ uniform on $[0.7, 1.3]$ and a nominal root-mean-square level $s$.
For the pressure source of the hydraulic subsystem, $s \in \{0.8, 1.8, 2.6\}$, and for its boundary flow, $s \in \{0.12, 0.35, 0.65\}$; the band is $[0.002, 0.010]$ or $[0.002, 0.035]$. For each of these two inputs, the three levels and the two bands form six groups, and the group labels are shuffled separately for the two inputs, so that neither the levels nor the bands are constrained to match within a trajectory. For the boundary velocity of the mechanical subsystem, $s \in \{0.20, 0.56, 1.00\}$ with the band $[0.005, 0.020]$.
The reference trajectories of the composed system are generated as in Task~3.
The coupling variables follow from differentiating the force balance $A_\mathrm{p} P + F_\mathrm{b} = 0$ with $P = 0.8\, q_h$ and $F_\mathrm{b} = -e_1(q_1)$: $v_\mathrm{b} = (0.8 A_\mathrm{p} Q_1 + e_1'(q_1) v_1) / (0.8 A_\mathrm{p}^2 + e_1'(q_1))$ and $Q_\mathrm{b} = A_\mathrm{p} v_\mathrm{b}$, which are substituted into the ODE of the seven states; the initial state satisfies the force balance.
The five trajectories of 100 steps of the composed system, provided for the identification of the area and for the scales of the initialization, start from states consistent with the force balance and within the training ranges.

\paragraph{Test trajectories.}
The test trajectories of the composed system are generated by the rule of Task~3 (see Appendix~\ref{app:task3}), where the residual of the coupling constraints is $\max_k |A_\mathrm{p} P(t_k) + F_\mathrm{b}(t_k)|$.

\paragraph{DINEs.}
The reservoir is assumed to be linear with an unknown coefficient, and the fluid resistor is nonlinear; the I-type elements, springs, and dampers are treated as in Task~3.
The scale of the boundary pressure of the hydraulic subsystem is fixed from the learned interconnection as in Task~3, whereas the boundary force of the mechanical subsystem is calibrated from measurements as described in Appendix~\ref{app:coupling}, with an intercept in addition.
For the mechanical subsystem, $N_{\mathrm{cal}}=4$ observations of the boundary force, preselected from the training data of the subsystem by a fixed random draw, are used.
The area is identified from four fixed states of the five trajectories of the composed system by the closed-form least squares of Appendix~\ref{app:coupling}, $\hat{A}_\mathrm{p} = -\sum_j \hat{P}_j \hat{F}_{\mathrm{b},j} / \sum_j \hat{P}_j^2$, where $\hat{P}_j$ and $\hat{F}_{\mathrm{b},j}$ are the calibrated boundary outputs of the two subsystems at the $j$-th state.
The hydraulic boundary output $\hat{P}$ is evaluated with the boundary flow set to zero.
The observations of the composed system do not determine the initial rates. In the initial-state estimation, the observation error is therefore normalized by the sample standard deviations $\sigma_j$ of the observed channels over the five trajectories, and a penalty $\frac{10^{-6}}{2} \sum_j (\dot{\hat{y}}_{0,j} / \rho_j)^2$ on the initial rates of the observed channels is added, where $\rho_j$ are the sample standard deviations of the increments divided by $\Delta{t}$ over the same trajectories.

\paragraph{Coupled Neural ODEs.}
Coupled Neural ODEs are as in Task~3 (see Appendix~\ref{app:task3}), except that the boundary network predicts the boundary pressure of the hydraulic subsystem, whose boundary input is the flow, and the boundary force of the mechanical subsystem, whose boundary input is the velocity, and that, in the composed system, the boundary flow and velocity are unknown inputs, and the boundary networks supply the pressure and the force of the coupling constraints of the piston.
If the connection has a transformer with an unknown ratio, the ratio is identified from the same states of the composed system as for DINEs by the same closed-form least squares, with the outputs of the frozen boundary networks, and is then frozen.
The same normalized initial-state objective and rate penalty are used for Coupled Neural ODEs.

\subsection{Task 5: reuse of the learned spring}
\label{app:reuse}
\paragraph{Isolation.}
The donor is DINEs trained on subsystem A of Task~3, and all five training seeds are used.
The energy network of the second spring, including its quadratic term, is isolated and frozen as the constitutive relation $\hat{e}(q) = \partial_q \hat{H}(q)$.
The scale of the force is fixed from the learned interconnection as described in Appendix~\ref{app:coupling}, since the coefficients of the junctions have unit magnitude and the subsystem contains no unknown transformer ratio.
With the learned flow of the source port written in the observed velocities of the masses as $f_s = \sum_i \omega_i v_i$, the coefficient of the largest magnitude is selected for each seed, and the scale is
\begin{equation*}\textstyle
    i^\star = \operatorname*{arg\,max}_i |\omega_i|,
    \qquad
    \alpha = 1 / |\omega_{i^\star}|,
\end{equation*}
so that the inertia connected to the boundary is identified from the learned interconnection rather than prescribed by its index.
The reused constitutive relation is $g(q) = \alpha \hat{e}(q)$; the network and its displacement axis are unchanged, and no offset of the force is corrected.
Nothing else is transferred from the donor, neither $(F, E)$, $\chi_\eta$, the other springs, the dampers, nor the masses, and neither additional measurements of the boundary variable or of the spring force, nor the true constitutive relation of the spring, nor the response of the receiving system is used to determine the scale.
\paragraph{Receiving system.}
The receiving system is known: a wall is connected to a main mass $M$ by a linear spring $k_\mathrm{h}$ with a damper $c_\mathrm{h}$ in parallel, the input $u(t)$ acts on $M$, and an auxiliary mass $m_\mathrm{a}$ is connected to $M$ by the reused spring with a damper $c_\mathrm{a}$ in parallel.
With the positions $x$ and $x_\mathrm{a}$ measured from the wall and the relative displacement $r = x_\mathrm{a} - x$ of the reused spring,
\begin{equation*}
    \begin{aligned}
        M \ddot{x}                       & = -k_\mathrm{h} x - c_\mathrm{h} \dot{x} + g(r) + c_\mathrm{a} \dot{r} + u(t), \\
        m_\mathrm{a} \ddot{x}_\mathrm{a} & = -g(r) - c_\mathrm{a} \dot{r},
    \end{aligned}
\end{equation*}
where $g(r)$ is the derivative of the scaled energy of the spring with respect to $r$.
The observed channels are $y = (x, x_\mathrm{a}, v, v_\mathrm{a})$ with $v = \dot{x}$ and $v_\mathrm{a} = \dot{x}_\mathrm{a}$.
All coefficients are fixed, $M = 2$, $m_\mathrm{a} = 0.3$, $k_\mathrm{h} = 2$, $c_\mathrm{h} = 0.04$, and $c_\mathrm{a} = 0.02$, and all initial positions and velocities are zero.
The prediction is terminated if the argument of the reused constitutive relation, including its evaluations within the Newton iterations, leaves the range of the displacements of the second spring in the training data of the donor; the argument is never clamped to this range.

\paragraph{Data.}
The reference trajectories use the true relation of the second spring of subsystem A of Task~3, $e_\star(q) = 0.3 q + 0.1 q^3$, in place of $g$ in the equations of the receiving system, and are generated and checked as in Appendix~\ref{app:common}.
Each input is a sum of five cosines, $u(t) = \sum_{\ell=1}^{5} A_\ell \cos(2\pi \nu_\ell\, t / \Delta t + \phi_\ell)$, with one frequency drawn uniformly from each of the bands $[0.005, 0.010]$, $[0.010, 0.015]$, $[0.015, 0.020]$, $[0.020, 0.025]$, and $[0.025, 0.030]$, phases uniform on $[0, 2\pi]$, and amplitudes $A_\ell = s\, b_\ell / \sqrt{\sum_\ell b_\ell^2 / 2}$ with relative amplitudes $b_\ell$ uniform on $[0.7, 1.3]$ and a nominal root-mean-square level $s \in \{0.05, 0.15\}$.
At each level, candidates are generated in order, at most 500 per level, and the first five whose reference responses keep the deformation of the reused spring within the interval between the first and the 99th percentiles of the displacements of the second spring in the training data of the donor are retained; no candidate was rejected.
This selection uses the reference system, not the prediction errors, and the same ten inputs are used for all five seeds of the donor.
Each test trajectory has 10{,}000 steps, as in the other tasks.

\section{Additional results and discussions}
\label{app:additional}
\paragraph{Prior knowledge given to the models.}
Table~\ref{tab:prior} summarizes the prior knowledge and assumptions given to the models.
The available prior knowledge is not identical across the models, but the comparison was performed under the same data and evaluation procedures.

\begin{table}[t]
    \centering
    \caption{Prior knowledge given to the models. The last row is used for scoring only.}
    \label{tab:prior}
    \scriptsize
    \setlength{\tabcolsep}{4pt}
    \begin{tabularx}{\linewidth}{>{\hsize=2.0\hsize}Y >{\hsize=0.8\hsize}Y >{\hsize=0.8\hsize}Y >{\hsize=0.8\hsize}Y >{\hsize=0.8\hsize}Y >{\hsize=0.8\hsize}Y}
        \toprule
                                                                               & DINEs                                 & Neural ODEs & PoDiNNs         & CT-SUBNET  & Coupled Neural ODEs         \\
        \midrule
        Training observations and known inputs                                 & yes                                   & yes         & yes             & yes        & yes                         \\
        Types and numbers of elements                                          & yes                                   & no          & C-type, I-type  & no         & dimensions only             \\
        Correspondence of observations to elements                             & yes                                   & as states   & yes             & as outputs & states, boundary port       \\
        Linearity of elements (coefficients unknown)                           & yes                                   & no          & I-type elements & no         & no                          \\
        $(n_f, n_e)$                                                           & yes                                   & no          & no              & no         & no                          \\
        True wiring and coefficients                                           & no                                    & no          & no              & no         & no                          \\
        Unobserved true states                                                 & no                                    & --          & --              & no         & --                          \\
        Boundary records of the subsystem training data                        & calibration (Task 4, mechanical only) & --          & --              & --         & boundary-network training   \\
        Connection of the boundary ports at composition                        & yes                                   & --          & --              & --         & yes                         \\
        States of the composed system (Task 4)                                 & area, initialization scales           & --          & --              & --         & area, initialization scales \\
        Displacement--force measurements of the true spring (Task 5)           & no                                    & --          & --              & --         & --                          \\
        Test observations after the first sample (after the history on Task 2) & scoring                               & scoring     & scoring         & scoring    & scoring                     \\
        \bottomrule
    \end{tabularx}
\end{table}

\paragraph{VPT for further thresholds.}
\begin{table}[t]
    \centering
    \caption{VPT of Tasks 1--5 for further thresholds $\varepsilon$.}
    \label{tab:thresholds}
    \scriptsize
    \setlength{\tabcolsep}{4pt}
    \begin{tabular}{llcccc}
        \toprule
        \textbf{Task} & \textbf{Model}         & $\varepsilon = 10^{-6}$                                     & $10^{-5}$                                     & $10^{-4}$                                     & $10^{-3}$                                     \\
        \midrule
        1             & Neural ODEs            & $0.008 \pm 0.001$                                           & $0.023 \pm 0.004$                             & $0.062 \pm 0.015$                             & $0.160 \pm 0.046$                             \\
        1             & PoDiNNs                & $0.000 \pm 0.000$                                           & $0.000 \pm 0.000$                             & $0.000 \pm 0.000$                             & $0.001 \pm 0.000$                             \\
        1             & DINEs (soft)$^\dagger$ & $0.270 \pm 0.160$ & $0.548 \pm 0.102$ & $0.898 \pm 0.022$ & $1.000 \pm 0.000$ \\
        1             & DINEs (ours)           & $0.204 \pm 0.028$                                           & $0.525 \pm 0.031$                             & $0.860 \pm 0.043$                             & $0.988 \pm 0.011$                             \\
        \midrule
        2             & CT-SUBNET ($n_z = 4$)  & $0.000 \pm 0.000$                                           & $0.000 \pm 0.000$                             & $0.039 \pm 0.011$                             & $0.681 \pm 0.036$                             \\
        2             & CT-SUBNET ($n_z = 8$)  & $0.000 \pm 0.000$                                           & $0.001 \pm 0.000$                             & $0.096 \pm 0.084$                             & $0.773 \pm 0.037$                             \\
        2             & CT-SUBNET ($n_z = 12$) & $0.000 \pm 0.000$                                           & $0.000 \pm 0.000$                             & $0.157 \pm 0.023$                             & $0.759 \pm 0.048$                             \\
        2             & CT-SUBNET ($n_z = 16$) & $0.000 \pm 0.000$                                           & $0.001 \pm 0.001$                             & $0.205 \pm 0.055$                             & $0.777 \pm 0.033$                             \\
        2             & CT-SUBNET ($n_z = 20$) & $0.000 \pm 0.000$                                           & $0.002 \pm 0.001$                             & $0.241 \pm 0.018$                             & $0.800 \pm 0.048$                             \\
        2             & DINEs (soft)           & (not converged)                                                                                                                                                                 \\
        2             & DINEs (ours)           & $0.000 \pm 0.000$                                           & $0.001 \pm 0.000$                             & $0.711 \pm 0.002$                             & $0.904 \pm 0.001$                             \\
        \bottomrule
        \addlinespace[2ex]
        \toprule
        \textbf{Task} & \textbf{Model}         & $\varepsilon = 10^{-8}$                                     & $10^{-7}$                                     & $10^{-6}$                                     & $10^{-5}$                                     \\
        \midrule
        3             & Coupled Neural ODEs    & $0.000 \pm 0.000$                                           & $0.000 \pm 0.000$                             & $0.002 \pm 0.001$                             & $0.020 \pm 0.005$                             \\
        3             & DINEs (ours)           & $0.003 \pm 0.001$               & $0.032 \pm 0.018$ & $0.970 \pm 0.068$ & $1.000 \pm 0.000$ \\
        \midrule
        4             & Coupled Neural ODEs    & $0.000 \pm 0.000$                                           & $0.001 \pm 0.000$                             & $0.004 \pm 0.001$                             & $0.268 \pm 0.156$                             \\
        4             & DINEs (ours)           & $0.002 \pm 0.002$               & $0.121 \pm 0.228$ & $0.831 \pm 0.321$ & $1.000 \pm 0.000$ \\
        \midrule
        5             & DINEs (ours)           & $0.001 \pm 0.001$               & $0.189 \pm 0.164$ & $0.968 \pm 0.050$ & $1.000 \pm 0.000$ \\
        \bottomrule
    \end{tabular}\\\raggedright
    $^\dagger$~3 of 5 seeds did not converge; averaged over the remaining 2.
\end{table}

Table~\ref{tab:thresholds} reports the VPT for further thresholds $\varepsilon$.
The superiority of DINEs over the other models is consistent.

\paragraph{Readout of the learned interconnection.}
We show examples of the learned $W_f$ and $W_e$ on Tasks~1 and~2 and read the interconnection from them.
For each matrix, linearly independent columns are chosen as the pivot columns and moved to the front, and the matrix is brought to reduced row echelon form, in which the pivot columns form the identity matrix.
Since neither system contains transformers, the nonzero coefficients should be $\pm 1$.
The coefficients of the elements whose port variables are unobserved, the resistive elements in this case, have gauge freedoms, as mentioned in Appendix~\ref{app:learning}.
These gauge freedoms are eliminated by normalizing the coefficients to $\pm 1$.
In principle, a very small coefficient could be canceled in scale by a very large flow, effort, or constitutive relation of the element (that is, the resistance) and accidentally removed; this does not happen thanks to the regularization term $\mathcal{L}_{\mathrm{var}}$ (as defined in Appendix~\ref{app:learning}).
Since the learned coefficients are continuous values, those with absolute values of at most $10^{-2}$ are treated as zero.

\textit{Task~1: three-way junction.}
The ports are ordered as $(\mathrm{C}_1, \mathrm{C}_2, \mathrm{C}_3, \mathrm{I}_1, \mathrm{I}_2, \mathrm{I}_3)$, the learned flow constraint matrix and its reduced row echelon form are
\begin{equation*}
    W_f \simeq
    \begin{blockarray}{cccccc}
        \scriptstyle\mathrm{C}_1 & \scriptstyle\mathrm{C}_2 & \scriptstyle\mathrm{C}_3 & \scriptstyle\mathrm{I}_1 & \scriptstyle\mathrm{I}_2 & \scriptstyle\mathrm{I}_3 \\
        \begin{block}{[rrrrrr]}
            0.572  & -0.355 & -0.217 & 0.572  & -0.355 & -0.217 \\
            -0.080 & -0.455 & 0.535  & -0.080 & -0.455 & 0.535  \\
        \end{block}
    \end{blockarray},
\end{equation*}
\begin{equation*}
    \operatorname{rref}(W_f) \simeq
    \begin{blockarray}{cccccc}
        \scriptstyle\mathrm{C}_1 & \scriptstyle\mathrm{C}_2 & \scriptstyle\mathrm{C}_3 & \scriptstyle\mathrm{I}_1 & \scriptstyle\mathrm{I}_2 & \scriptstyle\mathrm{I}_3 \\
        \begin{block}{[rrrrrr]}
            1.000 & 0.000 & -1.000 & 1.000 & 0.000 & -1.000 \\
            0.000 & 1.000 & -1.000 & 0.000 & 1.000 & -1.000 \\
        \end{block}
    \end{blockarray}.
\end{equation*}
Substituting $f_{\mathrm{C}_i} = -\dot{q}_i$ and $f_{\mathrm{I}_i} = v_i$ into the rows gives
\begin{equation*}
    \dot{q}_i = v_i - v_J \qquad (i = 1, 2, 3), \qquad v_J := v_1 - \dot{q}_1.
\end{equation*}
This implies that each spring $i$ with the velocity $\dot{q}_i$ is connected between the same junction with the velocity $v_J$ and its own mass with the velocity $v_i$.

The ports are ordered as $(\mathrm{C}_3, \mathrm{I}_1, \mathrm{I}_2, \mathrm{I}_3, \mathrm{C}_1, \mathrm{C}_2)$ so that four independent columns come first, the learned effort constraint matrix and its reduced row echelon form are
\begin{equation*}
    W_e \simeq
    \begin{blockarray}{cccccc}
        \scriptstyle\mathrm{C}_3 & \scriptstyle\mathrm{I}_1 & \scriptstyle\mathrm{I}_2 & \scriptstyle\mathrm{I}_3 & \scriptstyle\mathrm{C}_1 & \scriptstyle\mathrm{C}_2 \\
        \begin{block}{[rrrrrr]}
            0.456  & -0.431 & 0.282  & -0.522 & 0.365 & -0.348 \\
            -0.036 & -0.371 & -0.153 & 0.359  & 0.694 & 0.476  \\
            0.336  & 0.496  & 0.674  & 0.374  & 0.214 & 0.035  \\
            -0.587 & -0.312 & 0.330  & 0.354  & 0.079 & -0.563 \\
        \end{block}
    \end{blockarray},
\end{equation*}
\begin{equation*}
    \operatorname{rref}(W_e) \simeq
    \begin{blockarray}{cccccc}
        \scriptstyle\mathrm{C}_3 & \scriptstyle\mathrm{I}_1 & \scriptstyle\mathrm{I}_2 & \scriptstyle\mathrm{I}_3 & \scriptstyle\mathrm{C}_1 & \scriptstyle\mathrm{C}_2 \\
        \begin{block}{[rrrrrr]}
            1.000 & 0.000 & 0.000 & 0.000 & 1.000  & 1.000  \\
            0.000 & 1.000 & 0.000 & 0.000 & -1.000 & 0.000  \\
            0.000 & 0.000 & 1.000 & 0.000 & 0.000  & -1.000 \\
            0.000 & 0.000 & 0.000 & 1.000 & 1.000  & 1.000  \\
        \end{block}
    \end{blockarray}.
\end{equation*}
The rows give
\begin{equation*}
    e_{\mathrm{C}_1} + e_{\mathrm{C}_2} + e_{\mathrm{C}_3} = 0, \qquad e_{\mathrm{I}_i} = e_{\mathrm{C}_i} \quad (i = 1, 2, 3).
\end{equation*}
This implies the force balance at the junction and the action and reaction forces between the springs $\mathrm{C}_i$ and their masses $\mathrm{I}_i$.

\textit{Task~2: RC ladder.}
The flows of three of the seven resistive ports are unobserved.
The ports are ordered as $(\mathrm{C}_4, \mathrm{R}_1, \ldots, \mathrm{R}_7 \mid \mathrm{Se}_1, \mathrm{Se}_2, \mathrm{C}_1, \mathrm{C}_2, \mathrm{C}_3)$, where the resistors are numbered by their ports in the learned model and the vertical line separates the pivot columns, the learned flow constraint matrix is
\begingroup\small\makeatletter\BA@colsep=2pt\makeatother
\begin{equation*}
    W_f \simeq
    \begin{blockarray}{ccccccccccccc}
        \scriptstyle\mathrm{C}_4 & \scriptstyle\mathrm{R}_1 & \scriptstyle\mathrm{R}_2 & \scriptstyle\mathrm{R}_3 & \scriptstyle\mathrm{R}_4 & \scriptstyle\mathrm{R}_5 & \scriptstyle\mathrm{R}_6 & \scriptstyle\mathrm{R}_7 & \scriptstyle\mathrm{Se}_1 & \scriptstyle\mathrm{Se}_2 & \scriptstyle\mathrm{C}_1 & \scriptstyle\mathrm{C}_2 & \scriptstyle\mathrm{C}_3 \\
        \begin{block}{[rrrrrrrr|rrrrr]}
            .586  & -.173 & .446  & -.264 & .054  & -.044 & .195  & .284  & -.173 & .189  & -.120 & .195  & .340  \\
            .061  & .322  & .049  & -.293 & -.308 & .238  & .477  & .117  & .006  & -.497 & -.249 & -.267 & -.180 \\
            -.169 & -.271 & .032  & .041  & -.368 & .198  & -.363 & .216  & -.216 & .041  & -.150 & -.563 & .392  \\
            -.209 & -.201 & .119  & .114  & .438  & -.189 & -.092 & .335  & -.488 & -.352 & -.291 & .036  & -.305 \\
            -.304 & .163  & .167  & -.103 & .316  & .630  & -.032 & -.305 & -.159 & -.059 & -.142 & .276  & .355  \\
            .263  & -.194 & -.565 & .360  & .266  & .282  & .261  & .292  & .122  & -.206 & .079  & -.010 & .280  \\
            -.200 & .137  & -.079 & -.367 & -.034 & -.044 & -.381 & .543  & .431  & -.197 & .011  & .338  & .138  \\
            .056  & -.204 & -.452 & -.426 & .153  & -.106 & -.019 & -.230 & .105  & .234  & -.640 & -.066 & -.036 \\
        \end{block}
    \end{blockarray},
\end{equation*}
\endgroup
where $\mathrm{R}_1$, $\mathrm{R}_2$, and $\mathrm{R}_7$ are unobserved.
Its reduced row echelon form is $\operatorname{rref}(W_f) = [\, I_8 \;\; B_f \,]$ with
\begin{equation*}
    B_f \simeq
    \begin{blockarray}{ccccc}
        \scriptstyle\mathrm{Se}_1 & \scriptstyle\mathrm{Se}_2 & \scriptstyle\mathrm{C}_1 & \scriptstyle\mathrm{C}_2 & \scriptstyle\mathrm{C}_3 \\
        \begin{block}{[rrrrr]}
            1.000  & 1.000  & 0.998  & 0.996  & 0.994  \\
            1.426  & 0.008  & 1.425  & 1.426  & 0.004  \\
            -0.708 & -0.001 & 0.002  & 0.003  & 0.000  \\
            0.001  & 0.002  & 0.995  & 0.006  & 0.006  \\
            -0.002 & -0.002 & -0.004 & 1.000  & -0.005 \\
            -0.003 & -0.004 & -0.001 & -0.004 & 0.992  \\
            -1.003 & -1.003 & -1.003 & -0.997 & -1.001 \\
            -0.004 & -0.701 & -0.003 & -0.004 & -0.001 \\
        \end{block}
    \end{blockarray}.
\end{equation*}
Omitting the small coefficients and scaling the remaining coefficients as $\pm 1$ according to their signs gives
\begin{equation*}
    \widehat{B}_f =
    \begin{blockarray}{ccccc}
        \scriptstyle\mathrm{Se}_1 & \scriptstyle\mathrm{Se}_2 & \scriptstyle\mathrm{C}_1 & \scriptstyle\mathrm{C}_2 & \scriptstyle\mathrm{C}_3 \\
        \begin{block}{[rrrrr]}
            1  & 1  & 1  & 1  & 1  \\
            1  & 0  & 1  & 1  & 0  \\
            -1 & 0  & 0  & 0  & 0  \\
            0  & 0  & 1  & 0  & 0  \\
            0  & 0  & 0  & 1  & 0  \\
            0  & 0  & 0  & 0  & 1  \\
            -1 & -1 & -1 & -1 & -1 \\
            0  & -1 & 0  & 0  & 0  \\
        \end{block}
    \end{blockarray}.
\end{equation*}
Each row represents a Kirchhoff current law at a node of the circuit.
With the currents $I_{\mathrm{E}_1} = f_{\mathrm{Se}_1}$, $I_{\mathrm{E}_2} = f_{\mathrm{Se}_2}$, and $I_{\mathrm{C}_i} = -f_{\mathrm{C}_i}$, the rows give
\begin{gather*}
    f_{\mathrm{R}_1} = -I_{\mathrm{E}_1} + I_{\mathrm{C}_1} + I_{\mathrm{C}_2}, \qquad f_{\mathrm{R}_2} = I_{\mathrm{E}_1}, \qquad f_{\mathrm{R}_7} = I_{\mathrm{E}_2}, \\
    f_{\mathrm{R}_3} = I_{\mathrm{C}_1}, \qquad f_{\mathrm{R}_4} = I_{\mathrm{C}_2}, \qquad f_{\mathrm{R}_5} = I_{\mathrm{C}_3}, \qquad f_{\mathrm{R}_6} = I_{\mathrm{C}_4}, \\
    I_{\mathrm{E}_1} + I_{\mathrm{E}_2} = I_{\mathrm{C}_1} + I_{\mathrm{C}_2} + I_{\mathrm{C}_3} + I_{\mathrm{C}_4}.
\end{gather*}

The ports are ordered as $(\mathrm{Se}_1, \mathrm{Se}_2, \mathrm{C}_1, \mathrm{C}_2, \mathrm{C}_3 \mid \mathrm{C}_4, \mathrm{R}_1, \ldots, \mathrm{R}_7)$, the learned effort constraint matrix is
\begingroup\small\makeatletter\BA@colsep=2pt\makeatother
\begin{equation*}
    W_e \simeq
    \begin{blockarray}{ccccccccccccc}
        \scriptstyle\mathrm{Se}_1 & \scriptstyle\mathrm{Se}_2 & \scriptstyle\mathrm{C}_1 & \scriptstyle\mathrm{C}_2 & \scriptstyle\mathrm{C}_3 & \scriptstyle\mathrm{C}_4 & \scriptstyle\mathrm{R}_1 & \scriptstyle\mathrm{R}_2 & \scriptstyle\mathrm{R}_3 & \scriptstyle\mathrm{R}_4 & \scriptstyle\mathrm{R}_5 & \scriptstyle\mathrm{R}_6 & \scriptstyle\mathrm{R}_7 \\
        \begin{block}{[rrrrr|rrrrrrrr]}
            -.031 & -.415 & -.085 & -.146 & .485  & .194  & .374  & -.021 & .083  & .147  & -.484 & -.192 & -.292 \\
            -.182 & .290  & -.306 & .080  & -.204 & .320  & .580  & -.128 & .305  & -.082 & .203  & -.323 & .202  \\
            .099  & -.429 & -.157 & .334  & -.179 & .333  & -.389 & .069  & .156  & -.336 & .178  & -.335 & -.300 \\
            .179  & .103  & -.462 & .295  & .230  & -.343 & -.020 & .127  & .456  & -.296 & -.226 & .343  & .072  \\
            -.611 & -.021 & .201  & .389  & .135  & -.091 & .029  & -.434 & -.203 & -.389 & -.135 & .091  & -.015 \\
        \end{block}
    \end{blockarray},
\end{equation*}
\endgroup
and its reduced row echelon form is $\operatorname{rref}(W_e) = [\, I_5 \;\; B_e \,]$ with
\begin{equation*}
    B_e \simeq
    \begin{blockarray}{cccccccc}
        \scriptstyle\mathrm{C}_4 & \scriptstyle\mathrm{R}_1 & \scriptstyle\mathrm{R}_2 & \scriptstyle\mathrm{R}_3 & \scriptstyle\mathrm{R}_4 & \scriptstyle\mathrm{R}_5 & \scriptstyle\mathrm{R}_6 & \scriptstyle\mathrm{R}_7 \\
        \begin{block}{[rrrrrrrr]}
            -1.000 & -1.426 & 0.708  & -0.001 & 0.002  & 0.003  & 1.003 & 0.004 \\
            -1.000 & -0.008 & 0.001  & -0.002 & 0.002  & 0.004  & 1.003 & 0.701 \\
            -0.998 & -1.425 & -0.002 & -0.995 & 0.004  & 0.001  & 1.003 & 0.003 \\
            -0.996 & -1.426 & -0.003 & -0.006 & -1.000 & 0.004  & 0.997 & 0.004 \\
            -0.994 & -0.004 & 0.000  & -0.006 & 0.005  & -0.992 & 1.001 & 0.001 \\
        \end{block}
    \end{blockarray}.
\end{equation*}

Omitting the small coefficients and scaling the remaining coefficients as $\pm 1$ according to their signs gives
\begin{equation*}
    \widehat{B}_e =
    \begin{blockarray}{cccccccc}
        \scriptstyle\mathrm{C}_4 & \scriptstyle\mathrm{R}_1 & \scriptstyle\mathrm{R}_2 & \scriptstyle\mathrm{R}_3 & \scriptstyle\mathrm{R}_4 & \scriptstyle\mathrm{R}_5 & \scriptstyle\mathrm{R}_6 & \scriptstyle\mathrm{R}_7 \\
        \begin{block}{[cccccccc]}
            -1 & -1 & 1 & 0  & 0  & 0  & 1 & 0 \\
            -1 & 0  & 0 & 0  & 0  & 0  & 1 & 1 \\
            -1 & -1 & 0 & -1 & 0  & 0  & 1 & 0 \\
            -1 & -1 & 0 & 0  & -1 & 0  & 1 & 0 \\
            -1 & 0  & 0 & 0  & 0  & -1 & 1 & 0 \\
        \end{block}
    \end{blockarray}.
\end{equation*}
Each row represents a Kirchhoff voltage law around an independent loop of the circuit.
We can confirm $\widehat{B}_e = -\widehat{B}_f^\top$.
The rows give two common voltages
\begin{align*}
    V_L & := e_{\mathrm{Se}_1} + e_{\mathrm{R}_2} = e_{\mathrm{C}_1} - e_{\mathrm{R}_3} = e_{\mathrm{C}_2} - e_{\mathrm{R}_4}, \\
    V_R & := e_{\mathrm{Se}_2} + e_{\mathrm{R}_7} = e_{\mathrm{C}_3} - e_{\mathrm{R}_5} = e_{\mathrm{C}_4} - e_{\mathrm{R}_6},
\end{align*}
and $e_{\mathrm{R}_1} = V_L - V_R$.

From these relations, we can reconstruct the circuit diagram of the RC ladder, including the three resistors whose roles in the interconnection were unknown.
\end{document}